\documentclass[sigconf,nonacm]{acmart}

\DeclareMathOperator*{\argmax}{arg\,max}
\usepackage{amsthm}
\usepackage{amsfonts}
\usepackage{amsmath}
\usepackage{bm}
\usepackage{subfig}
\usepackage{graphicx}
\usepackage{caption}
\usepackage{thmtools}
\usepackage{thm-restate}
\usepackage{enumitem}
\usepackage[ruled,vlined,linesnumbered,lined,commentsnumbered,noend]{algorithm2e}
\usepackage{mathtools}

\DeclareMathOperator*{\argmin}{argmin}
\DeclarePairedDelimiter\ceil{\lceil}{\rceil}
\newcommand{\Mod}[1]{\ (\mathrm{mod}\ #1)}

\AtBeginDocument{%
  \providecommand\BibTeX{{%
    \normalfont B\kern-0.5em{\scshape i\kern-0.25em b}\kern-0.8em\TeX}}}

\copyrightyear{2021} 
\acmYear{2021} 
\setcopyright{acmcopyright}\acmConference[KDD '21]{Proceedings of the 27th ACM SIGKDD Conference on Knowledge Discovery and Data Mining}{August 14--18, 2021}{Virtual Event, Singapore}
\acmBooktitle{Proceedings of the 27th ACM SIGKDD Conference on Knowledge Discovery and Data Mining (KDD '21), August 14--18, 2021, Virtual Event, Singapore}
\acmPrice{15.00}
\acmDOI{10.1145/3447548.3467370}
\acmISBN{978-1-4503-8332-5/21/08}

\begin{document}
\settopmatter{printfolios=true}
\title{Q-Learning Lagrange Policies for Multi-Action Restless Bandits}


\author{Jackson A. Killian}
\email{jkillian@g.harvard.edu}
\orcid{0000-0001-8555-1327}
\affiliation{%
  \institution{Harvard University}
  \city{Cambridge}
  \state{MA}
  \country{USA}
}

\author{Arpita Biswas}
\email{arpitabiswas@seas.harvard.edu}
\orcid{0000-0002-5720-013X}
\affiliation{%
  \institution{Harvard University}
  \city{Cambridge}
  \state{MA}
  \country{USA}
}

\author{Sanket Shah}
\email{sanketshah@g.harvard.edu}
\orcid{0000-0002-0632-7966}
\affiliation{%
  \institution{Harvard University}
  \city{Cambridge}
  \state{MA}
  \country{USA}
}

\author{Milind Tambe}
\email{milind_tambe@harvard.edu}
\orcid{0000-0003-3296-3672}
\affiliation{%
  \institution{Harvard University}
  \city{Cambridge}
  \state{MA}
  \country{USA}
}



\begin{abstract}
Multi-action restless multi-armed bandits (RMABs) are a powerful framework for constrained resource allocation in which $N$ independent processes are managed. However, previous work only study the offline setting where problem dynamics are known. We address this restrictive assumption, designing the first algorithms for learning good policies for Multi-action RMABs online using combinations of Lagrangian relaxation and Q-learning. Our first approach, MAIQL, extends a method for Q-learning the Whittle index in binary-action RMABs to the multi-action setting. We derive a generalized update rule and convergence proof and establish that, under standard assumptions, MAIQL converges to the asymptotically optimal multi-action RMAB policy as $t\rightarrow{}\infty$. However, MAIQL relies on learning Q-functions and indexes on two timescales which leads to slow convergence and requires problem structure to perform well. Thus, we design a second algorithm, LPQL, which learns the well-performing and more general Lagrange policy for multi-action RMABs by learning to minimize the Lagrange bound through a variant of Q-learning. To ensure fast convergence, we take an approximation strategy that enables learning on a single timescale, then give a guarantee relating the approximation's precision to an upper bound of LPQL's return as $t\rightarrow{}\infty$. Finally, we show that our approaches always outperform baselines across multiple settings, including one derived from real-world medication adherence data.
\end{abstract}

\ccsdesc[500]{Computing methodologies~Reinforcement learning}

\keywords{Multi-action Restless Bandits; Q-learning; Lagrangian Relaxation}


\maketitle

\section{Introduction}
\textit{Restless Multi-Armed Bandits} (RMABs) are a versatile sequential decision making framework in which, given a budget constraint, a planner decides how to allocate resources among a set of independent processes that evolve over time. This model, diagrammed in Fig.~\ref{fig:schematic}, has wide-ranging applications, such as in healthcare~~\cite{lee2019optimal,adityamate2020collapsing,bhattacharya2018restless}, anti-poaching patrol planning~\cite{qian2016restless}, sensor monitoring tasks~\cite{Ianello2012,glazebrook2006some}, machine replacement \cite{ruiz2020multi}, and many more. However, a key limitation of these approaches is they only allow planners a binary choice---whether or not to allocate a resource to an arm at each timestep. However, in many real world applications, a planner may choose among multiple actions, each with varying cost and providing varying benefits. For example, in a community health setting (e.g., Figure~\ref{fig:schematic}), a health worker who monitors patients' adherence to medication may have the ability to provide interventions via text, call, or in-person visit. Such \textit{multi-action} interventions require varying amount of effort (or cost), and cause varying effects on patients' adherence. Given a fixed budget, the problem for a health worker is to decide what interventions to provide to each patient and when, with the goal of maximizing the overall positive effect (e.g., the improvement of patients' adherence to medication). 

Owing to the improved generality of \textit{multi-action RMABs}
over binary-action RMABs, this setting has gained attention in recent years~\cite{glazebrook2011general,hodge2015asymptotic,killian2021multiAction}. However, critically, all these papers have assumed the \textit{offline} setting, in which the dynamics of all the underlying processes are assumed to be known before planning. This assumption is restrictive since, in most cases, the planner will not have perfect information of the underlying processes, for example, how well a patient would respond to a given type of intervention. 


To address this shortcoming in previous work, this paper presents the first algorithms for the \textit{online} setting for multi-action RMABs. Indeed, the online setting for even binary-action RMABs has received only limited attention, in the works of Fu et al.~\cite{fu2019towards},  Avrachenkov and Borkar \cite{avrachenkov2020whittle}, and Biswas et al.~\cite{biswas2021learn,biswas2021learning}. These papers adopt variants of the Q-learning update rule~\cite{watkins1989learning,watkins1992q}, a well studied reinforcement learning algorithm, for estimating the effect of each action across changing dynamics of the systems. These methods aim to learn Whittle indexes~\cite{whittle1988restless} over time and use them for choosing actions. In the offline version, it has been shown that these indexes lead to an optimal selection policy when the RMAB instances meet the \textit{indexability} condition. However, these methods only apply to binary-action RMABs. Our paper presents two new algorithms for online multi-action RMABs to address the shortcomings of previous work and presents an empirical comparison of the approaches. The paper provides three key contributions: 

\begin{figure*}[!ht]
\includegraphics[width=0.93\textwidth]{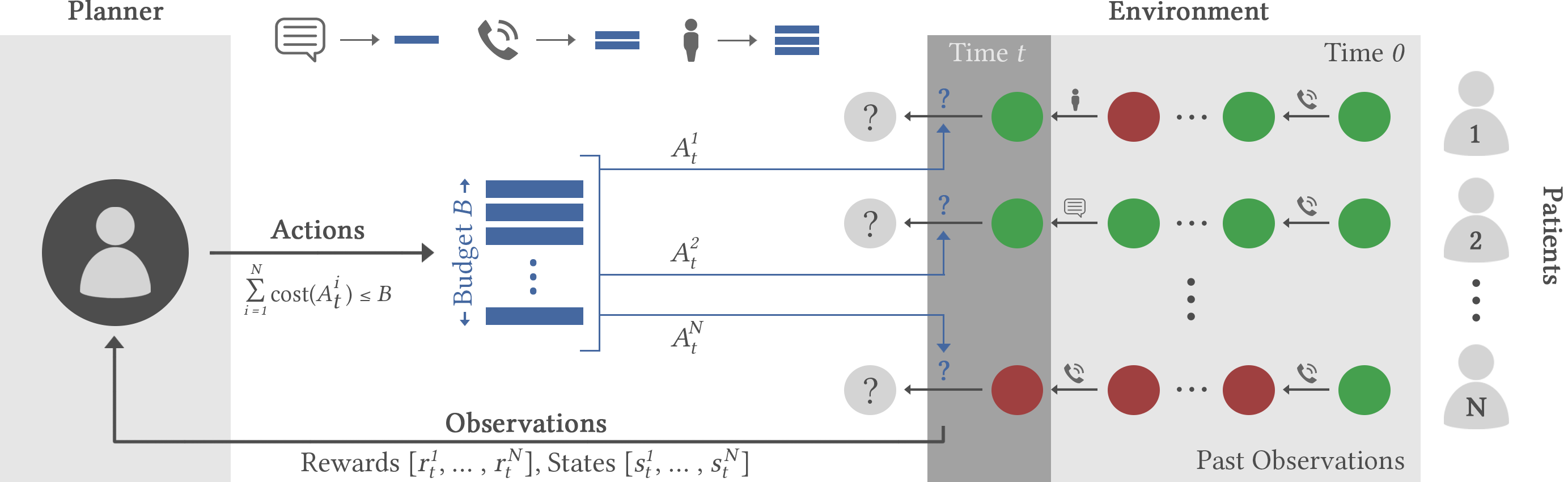}
\centering
\caption{Schematic of a multi-action RMAB. At each timestep, $t$, the planner (e.g., health worker) takes one action on each of $N$ processes (e.g., patients). The sum cost of actions each timestep must not exceed a budget, $B$. After taking actions at each timestep, the planner observes the rewards and state transitions of the processes, which the planner uses to improve their action selection in the future. The goal is to maximize reward.}
\label{fig:schematic}
\end{figure*}
\begin{enumerate}[leftmargin=*,topsep=0.5em]
    \item \textbf{We design Multi-action Index-based Q-learning (MAIQL)}. We consider a multi-action notion of indexability where the index for each action represents the ``fair charge'' for taking that action \cite{glazebrook2011general}. If the dynamics of the underlying systems were known beforehand, an optimal policy for multi-action indexable RMABs would choose actions based on these indexes when a linear structure on the action costs is assumed \cite{hodge2015asymptotic}. We establish that, when these dynamics are unknown and are required to be learned over time, MAIQL provably converges to these indexes for any multi-action RMAB instance following the assumptions on cost and indexability. However, these assumptions can be limiting, and in addition, the algorithm requires a two-timescale learning procedure that can be slow and unstable.
    \item \textbf{We propose a more general algorithm, Lagrange Policy Q-learning (LPQL)}. This method takes a holistic back-to-the-basics approach of analyzing the Lagrangian relaxation of the multi-action RMAB problem and learning to play the \textit{Lagrange policy} using the estimated Q values which are updated over time. This policy converges more quickly than MAIQL and other benchmark algorithms, is applicable to problems with arbitrary cost structures, and does not require the indexability condition. 
    \item \textbf{We demonstrate the effectiveness of MAIQL and LPQL as compared to various baselines on several experimental domains,} including two synthetically generated domains and derived from a real-world dataset on medication adherence of tuberculosis patients. Our algorithms converge to the state-of-the-art offline policy much faster than the baselines, taking a crucial step toward real-world deployment in online settings.\footnote{Code available at: \url{https://github.com/killian-34/MAIQL_and_LPQL}}
\end{enumerate}

\section{Related Work}
The \textit{restless multi-armed bandit} (RMAB) problem was introduced by Whittle~\cite{whittle1988restless} where he showed that a relaxed version of the offline RMAB problem can be solved optimally using a heuristic called the \textit{Whittle index policy}. This policy is shown to be asymptotically optimal when the RMAB instances satisfy the \textit{indexability} condition \cite{weber1990index}. Moreover, Papadimitriou and Tsitsiklis~\cite{papadimitriou1994complexity} established that solving RMABs is PSPACE-hard, even for the special case when the transition rules are deterministic. 

Since then, a vast literature have studied various subclasses of RMABs and provided algorithms for computing the Whittle index. Lee et al.~\cite{lee2019optimal} study the problem of selecting patients for screening with the goal of maximizing early-stage cancer detection under limited resources.  Mate et al.~\cite{adityamate2020collapsing} consider bandits with two states to model a health intervention problem, where the uncertainty collapses after an active action. They showed that the model is indexable and gave a mechanism for computing the Whittle index policy. Bhattacharya~\cite{bhattacharya2018restless} models the problem of maximizing the coverage and spread of health information with limited resources as an RMAB and proposes a hierarchical policy. Similarly, several other papers \cite{glazebrook2006some,Sombabu2020,Liu2010} give Whittle indexability results for different subclasses of RMABs where there are only two possible actions.

For more than two actions, Glazebrook et al.~\cite{glazebrook2011general,hodge2015asymptotic} extended Whittle indexability to multi-action RMABs where the instances are assumed to have special monotonic structure. Along similar lines, Killian et al.~\cite{killian2021multiAction} proposed a method that leverages the convexity of an approximate Lagrangian version of the multi-action RMAB problem. 

Also related to multi-action RMABs are weakly coupled Markov decision processes (WCMDP). The goal of a WCMDP is to maximize reward subject to a set of constraints over actions, managing a finite number of independent Markov decision processes (MDPs). Hawkins~\cite{hawkins2003langrangian} studied a Lagrangian relaxation of WCMDPs and proposed an LP for minimizing the Lagrange bound. On the other hand, Adelman and Mersereau~\cite{adelman2008relaxations} provide an approximation algorithm that achieves a tighter bound than the Lagrange approach to WCMDPs, trading off scalability. A more scalable approximation method is provided by Gocgun and Ghate~\cite{gocgun2012lagrangian}.

    
  

However, these papers focused only on the offline versions of the problem in which the dynamics (transition and observation models) are known apriori. In the online setting, there has been some recent work on binary-action RMABs. Gafni and Cohen~\cite{gafni2020learning} propose an algorithm that learns to play the arm with the highest \emph{expected} reward. However, this is suboptimal for general RMABs since rewards are state- and action-dependent. 
Addressing this, Biswas et al.~\cite{biswas2021learn} give a Q-learning-based based algorithm 
that acts on the arms that have the largest difference between their active and passive Q values. Fu et al.~\cite{fu2019towards} take a related approach that adjust the Q values by some $\lambda$, and use it to estimate the Whittle index. Similarly, Avrachenkov and Borkar \cite{avrachenkov2020whittle} provide a two-timescale algorithm that learns the Q values as well as the index values over time. However, their convergence proof requires indexability and that all arms are homogeneous with the same underlying MDPs. We use the two-timescale methodology and define a multi-action indexability criterion to provide a general framework to learn multi-action RMABs with provable convergence guarantees. Our work is the first to address the multi-action RMAB setting online.

\section{Preliminaries and Notations}\label{sec:prelim}
A Multi-action RMAB instance consists of $N$ arms and a budget $B$ on the total cost. Each arm $i\in[N]$ follows an MDP \cite{puterman2014markov}. We define an MDP $\{ \mathcal{S}, \mathcal{A}, \mathcal{C}, r, T, \beta\}$ as a finite set of states $\mathcal{S}$, a finite set of $M$ actions $\mathcal{A}$, a finite set of action costs $\mathcal{C}:=\{c_j\}_{j\in \mathcal{A}}$, a reward function $r: \mathcal{S} \rightarrow \mathbb{R}$, a transition function $T(s, a, s^{\prime})$ denoting the probability of transitioning from state $s$ to state $s^\prime$ when action $a$ is taken, and a discount factor $\beta \in [0,1)$\footnote{$\beta$ is only included under the discounted reward case, as opposed to the average reward case which we address later.}. 
An MDP \emph{policy} $\pi: \mathcal{S} \rightarrow \mathcal{A}$ maps states to actions. The long-term \emph{discounted reward} of arm $i$ starting from state $s$ is defined as
\begin{equation}\label{eq:discounted_reward}
J_{\beta,\pi^i}^{i}(s) = E\left[\sum_{t=0}^\infty \beta^tr^i(s^i_{t})|\pi^i, s^i_0 = s\right]
\end{equation}
where $s^i_{t+1}\sim T(s^i_t, \pi^i(s^i_t), \cdot)$. For ease of exposition, we assume the action sets and costs are the same for all arms, but our methods will apply to the general case where each arm has arbitrary (but finite) state, action, and cost sets. Without loss of generality, we also assume that the actions are numbered in increasing order of their costs, i.e., $0=c_0 \leq c_1 \leq \ldots, c_M$. Now, the planner must take decisions for all arms jointly, subject to two constraints each round: (1)~select one action for each arm and (2)~the sum of action costs over all arms must not exceed a given budget $B$. Formally, the planner must choose a decision matrix $\bm{A} \in \{0,1\}^{N\times M}$ such that:
\begin{align}
    &\sum_{j=1}^{M} \bm{A}_{ij} = 1 \hspace{3mm} \forall i \in [N]
    \qquad
    &\sum_{i=1}^{N}\sum_{j=1}^{M} \bm{A}_{ij}c_j \le B \label{eq:budget_constraints}
\end{align}
Let $\overline{\mathcal{A}}$ be the set of decision matrices respecting the constraints in \ref{eq:budget_constraints} and let $\bm{s} = (s^1, ..., s^{N})$ represent the initial state of each arm. The planner's goal is to maximize the total discounted reward of all arms over time, subject to the constraints in \ref{eq:budget_constraints}, as given by the constrained Bellman equation:
\begin{equation}
\begin{aligned}\label{eq:combined_value_function}
    J(\bm{s}) = \max_{\bm{A}\in \overline{\mathcal{A}}}\left\{\sum_{i=1}^{N} r^i(s^i) + \beta E[J(\bm{s}^\prime) | \bm{s}, \bm{A}]\right\}
\end{aligned}
\end{equation}
However, this corresponds to an optimization problem with exponentially many states and combinatorially many actions, making it PSPACE-Hard to solve directly \cite{papadimitriou1994complexity}. To circumvent this, we take the Lagrangian relaxation of the second constraint in \ref{eq:budget_constraints} \cite{hawkins2003langrangian}:
\begin{equation}
\begin{aligned}\label{eq:relaxed_value_func}
    &J(\bm{s}, \lambda) = \\
    &\max_{\bm{A}}\left\{\sum_{i=1}^{N} r^i(s^i) + \lambda(B - \sum_{i=1}^{N}\sum_{j=1}^{M} \bm{A}_{ij}c_j) + \beta E[J(\bm{s}^\prime,\lambda) | \bm{s}, \bm{A}] \right\}
\end{aligned}
\end{equation}
Since this constraint was the only term coupling the MDPs, relaxing this constraint decomposes the problem except for the shared term $\lambda$. So Eq.~\ref{eq:relaxed_value_func} can be rewritten as (see \cite{adelman2008relaxations}):
\begin{equation}
\begin{aligned}\label{eq:decoupled_value_func}
    &J(\bm{s}, \lambda) = \frac{\lambda B}{1-\beta} + \sum_{i=1}^{N}\max_{a_j^i\in\mathcal{A}^i}\{(Q^i(s^i, a_j^i, \lambda)\}
\end{aligned}
\end{equation}
\begin{equation}
\begin{aligned}\label{eq:arm_value_func_lagrange}
    &\text{where } Q^i(s^i, a_j^i, \lambda) = \\
    &r^i(s^i) - \lambda c_j + \beta\sum_{s^{\prime}}T(s^i, a_j^i, s^{\prime}) \max_{a_j\in\mathcal{A}^i}\{(Q^i(s^{\prime}, a_j, \lambda)\}
\end{aligned}
\end{equation}
In Eq.~\ref{eq:arm_value_func_lagrange}, each arm is effectively decoupled, allowing us to solve for each arm independently for a given value of $\lambda$. The choice of $\lambda$, however, affects the resulting optimal policies in each of the arms. One intuitive interpretation of $\lambda$ is that of a ``penalty'' associated with acting -- given a fixed budget $B$, a planner must weigh the cost of acting $\lambda c_j$ against its ability to collect higher rewards. Thus, as $\lambda$ is increased, the optimal policies on each arm will tend to prefer actions that generate the largest "value for cost".

The challenges we address are two-fold: (1) How to learn policies online that can be tuned by varying $\lambda$ and (2) How to make choices for the setting of $\lambda$ that lead to good policies. Our two algorithms in Sections \ref{sec:MAIQL} and \ref{sec:LPQL} both build on Q-Learning to provide alternative ways of tackling these challenges -- where MAIQL builds on the rich existing literature of ``index'' policies, LPQL goes ``back to basics'' and provides a more fundamental approach based on the Lagrangian relaxation discussed above.
\newtheorem{assumption}{Assumption}

\section{Algorithm: MAIQL}\label{sec:MAIQL}
Our first algorithm will reason about $\lambda$'s influence on each arm's value function independently. Intuitively, this is desirable because it simplifies one size-$N$ problem to $N$ size-one problems that can be solved quickly. Our goal will be to compute \textit{indexes} for each action on each arm that capture a given action's value, then greedily follow the indexes as our policy. Such an \textit{index policy} was proposed by \citet{whittle1988restless} for binary-action RMABs, in which the index is a value of $\lambda$ such that the optimal policy is indifferent between acting and not acting in the given state. This policy has been shown to be asymptotically optimal under the indexability condition~\cite{weber1990index}.

Glazebrook et al.~\cite{glazebrook2011general} and \citet{hodge2015asymptotic} extended the definition and guarantees, respectively, of the Whittle index to multi-action RMABs that satisfy the following assumptions:
\begin{enumerate}
    \item Actions have equally spaced costs, i.e., after normalization, the action costs can be expressed as $\{0, 1, \ldots, M - 1\}$.
    \item The utility of acting is submodular in the cost of the action.
\end{enumerate}
For such multi-action RMABs, \citet{glazebrook2011general} defines multi-action indexability and the multi-action index as follows:
\begin{definition}[Multi-action indexability]\label{def:maindexability}
    An arm is multi-action indexable if, for every given action $a_j \in \mathcal{A}$, the set of states in which it is optimal to take an action of cost $c_j$ or above decreases monotonically from $\mathcal{S}\xrightarrow{}\varnothing$ as $\lambda$ increases from $-\infty \to \infty$.
\end{definition}
\begin{definition}[Multi-action index, $\lambda^*_{s, a_j}$]
    For a given state $s$ and action $a_j$, the multi-action index is the minimum $\lambda^*_{s,a_j}$ that is required to become indifferent between the actions $a_j$ and $a_{j-1}$:
    \begin{align}
        \lambda^*_{s, a_j} &= \inf_\lambda\{Q(s, a_{j}, \lambda) \leq Q(s, a_{j-1}, \lambda)\} \\
        &= \lambda, \; s.t. \; Q(s, a_{j}, \lambda) = Q(s, a_{j-1}, \lambda) \label{eqn:maindex}
    \end{align}
    where $Q(s, a_{j}, \lambda)$ is the Q-value of taking action $a_j$ in state $s$ with current and future rewards adjusted by $\lambda$.
\end{definition}

Given these multi-action indices, \citet{hodge2015asymptotic} suggest a way to greedily allocate units of resources that is asymptotically optimal -- assume that the arms are in some state $\bm{s}$, then iterate from $1 \ldots B$ and in each round allocate a unit of resource to the arm with the highest multi-action index associated with the next unit of resource. Specifically, if $\bm{\theta} = \langle \theta^1\ldots \theta^N \rangle$ units of resource have been allocated to each arm so far, then we allocate the next unit of resource to the arm with the highest $\lambda^*_{s^i, a_{\theta^i}}$. Given that the action utilities ($\lambda^*_{s, a_j}$) are submodular in $a_j$ by assumption, this \textit{multi-action index policy} leads to the allocation in which the sum of multi-action index values across all the arms are maximised.

Given the policy's theoretical guarantees, an index-based solution to the multi-action RMAB problem is attractive. The question then is how to calculate the value of the multi-action indices $\lambda^*_{s, a_j}$. In the online setting (when the RMAB dynamics are unknown apriori), \citet{avrachenkov2020whittle} proposes a method for estimating the Whittle indexes for binary-action RMABs and, in addition, proves that this algorithm's estimate converges to the Whittle index.

In this section, we describe the \textbf{M}ulti-\textbf{A}ction \textbf{I}ndex \textbf{Q}-\textbf{L}earning (MAIQL) algorithm. Our algorithm generalizes the update rule of the learning algorithm proposed by \citet{avrachenkov2020whittle}. We consider the notion of multi-action indexability from \citet{glazebrook2011general} to create an update rule that allows us to estimate the multi-action indexes (Section \ref{sec:maiqlalgo}). In addition, we use the multi-action indexability property to show that the convergence guarantees from \citet{avrachenkov2020whittle} are preserved in this multi-action extension (Section \ref{sec:maiqltheory}).

\subsection{Algorithm}\label{sec:maiqlalgo}
From Equation \ref{eqn:maindex}, we observe that if we could estimate the Q values for all the possible values of $\lambda$, we would know the value of $\lambda^*_{s, a_{j}}$. This is not possible in general, but we can convert this insight into an update rule for estimating  $\lambda^*_{s, a_{j}}$ in which we update the current estimate in the direction such that Eq.~\ref{eqn:maindex} is closer to being satisfied. Based on this, we propose an iterative scheme in which Q values and $\lambda^*_{s, a_{j}}$ are learned together.

An important consideration is that, because the Q and $\lambda^*_{s, a_{j}}$ values are inter-dependent, it is not straightforward to learn them together, since updating the estimate of one may adversely impact our estimate of the other. To combat this, we decouple the effects of learning each component by relegating them to separate time-scales. Concretely, this means that an adaptive learning rate $\alpha(t)$ for the Q values and $\gamma(t)$ for $\lambda$-values are chosen such that $\lim_{t \to \infty} \frac{\gamma(t)}{\alpha(t)} \to 0$, i.e., the Q values are learned on a fast-time scale in which $\lambda$ values can be seen as quasi-static (details in the appendix). The resultant two time-scale approach is given below.

To calculate the multi-action index, for a given state $s \in \mathcal{S}$ and action $a_j \in \mathcal{A}$, we store two sets of values: (1) the Q values for all states and actions, $Q(s, a_j)\; \forall s \in S,\, a_j \in \mathcal{A}$, and (2) the current estimate of the multi-action index $\lambda_{s, a_j}$. All the Q and $\lambda$ values are initiated to zero. Then, for a given state $s$ in which we take action $a_j$, we observe the resultant reward $r$ and next state $s^\prime$, then perform the following updates:
\begin{enumerate}[leftmargin=*]
    \item \textbf{Q-update:} At a fast time-scale (adjusted by $\alpha(\nu(s, a_j, t))$), update to learn the correct Q values as in standard Q-learning:
    \begin{align}\label{eqn:maqupdate}
        Q_{\lambda}^{t+1}(s, a_j) = Q_{\lambda}^t(s, a_j) + \alpha(&\nu(s, a_j, t)) \big [ [r(s) - \lambda^{t}_{s,a_j}c_j\nonumber\\
        - f(Q_{\lambda}^t) + 
        & \; \max_{a_j^\prime} Q_{\lambda}^t(s^\prime, a_j^\prime)] - Q_{\lambda}^t(s, a_j) \big ]
    \end{align}
    where $\nu(s, a_j, t)$ is a ``local-clock'' that stores how many times the specific $Q_{\lambda}(s, a_j)$ value has been updated in the past, and $f(Q_{\lambda}^t) = \frac{\sum_{s, a_j} Q_{\lambda}^t(s, a_j)}{\sum_{s, a_j} 1}$ is a function whose value converges to the optimal average reward \cite{abounadi2001learning}. We give the average reward case to align with the traditional derivation of binary-action Whittle indexes, but this update (and related theory) can be extended easily to the discounted reward case.
    \item \textbf{$\lambda$-update:} Then, at a slower time-scale (adjusted by a function $\gamma(t)$), we update the value of $\lambda^t_{s, a_j}$ according to:
    \begin{equation}\label{eqn:malamupdate}
        \lambda^{t+1}_{s,a_j} = \lambda^t_{s, a_j} + \gamma(t) \cdot (Q_{\lambda}^t(s, a_j) - Q_{\lambda}^t(s, a_{j-1}))
    \end{equation}
\end{enumerate}

Note that the updates described in the paragraph above correspond to the estimation of a single multi-action index. To efficiently estimate $\lambda^*(s, a_j)\; \forall\, s, a$, we make use of the fact that our algorithm, like the Q-learning algorithm on which it is based, is off-policy -- an off-policy algorithm does not require collecting samples using the policy that is being learned. As a result, rather than learn each of these multi-action index values sequentially, we learn them in parallel based on the samples drawn from a single policy.

Specifically, since learning each index value requires imposing the current estimate $\lambda$ on all current and future action costs, and since a separate index is learned for all arms, states, and non-passive actions, $N(M-1)|\mathcal{S}|$ separate Q-functions (each a table of size $|\mathcal{S}|\times M$) and $\lambda$-values must be maintained, requiring $\mathcal{O}(NM^2|\mathcal{S}|^2)$ memory. However, since the estimation of each index is independent, each round, the index and its Q-function can be updated in parallel, keeping the process efficient, but requiring $\mathcal{O}(NM|\mathcal{S}|)$ time if computed in serial. To take actions, we follow an $\epsilon$-greedy version of the multi-action index policy -- which, when not acting randomly, greedily selects indices in increasing size order for each arm's current state, taking $\mathcal{O}(NM)$ time -- and store the resultant $\langle s,a,r,s'\rangle$ tuple in a replay buffer. The replay buffer is important because, in the multi-action setting, each $(s, a)$ pair is not sampled equally often; specifically, especially when $B$ is small, it is less likely to explore more expensive actions. After every fixed number of time-steps of running the policy, we randomly pick some $\langle s,a,r,s'\rangle$ tuples from the replay buffer with probability weighted inversely to the number of times the tuple has been used for training, and update the Q values associated with each of the multi-action indexes and the $\lambda_{s,a}$ estimate for the sampled $(s, a)$.\footnote{all algorithms in this paper will be equipped with the replay buffer for fairness of comparison.} The resulting algorithm is guaranteed to converge to the multi-action indexes. Pseudocode is given in the appendix.


\subsection{Theoretical Guarantees}\label{sec:maiqltheory}

The attractiveness of the MAIQL approach comes from the fact that, if the problem is multi-action indexable, the indexes can always be found. Formally, we show:

\begin{restatable}{theorem}{MAIQL}\label{thm:MAIQL}
    MAIQL converges to the optimal multi-action index $\lambda_{s, a}^*$ for a given state $s$ and action $a$ under Assumptions \ref{ass:unichain}, \ref{ass:async}, \ref{ass:alpha}, and the problem being multi-action indexable.
\end{restatable}

\begin{proof}[Proof Sketch]
\textbf{At the fast time-scale:} We can assume $\lambda_{s, a}$ to be static. Then, for a given value of $\lambda_{s, a} = \lambda'$, the problem reduces to a standard MDP problem, and the Q-learning algorithm converges to the optimal policy.

\textbf{At the slow time-scale:} We can consider the fast-time scale process to have converged, and we have the optimal Q values $Q^*_{\lambda'}$ corresponding to the current estimate of $\lambda_{s,a}$. Then, by the multi-action indexability property, we know that if $\lambda < \lambda^*$ an action of weight $a$ or higher is preferred. As a result, $Q^*_{\lambda_{s, a}}(s, a) - Q^*_{\lambda_{s, a}}(s, a-1) > 0$, and so $\lambda^{t+1} > \lambda^{t}$. When $\lambda > \lambda^*$, the opposite is true and so $\lambda^{t+1} < \lambda^{t}$. As a result, we constantly improve our estimate of $\lambda$ such that we eventually converge to the optimal multi-action index, i.e., $\lim_{t\to\infty} \lambda^{t} \to \lambda_{s, a}^*$.
\end{proof}

The detailed proof follows along the lines of \citet{avrachenkov2020whittle}, and can be found in the appendix. However, while they consider convergence in the binary-action case, our approach generalizes to the multi-action setting. The crux of the proof lies in showing how the multi-action index generalizes the properties of the Whittle index in the multi-action case, and leads to convergence in the slow time-scale.


    
\subsection{MAIQL Limitations}
The main limitations of MAIQL are (1) it assumes multi-action indexability and equally-spaced action costs to be optimal and (2) it learns on two time-scales, making convergence slow and unstable in practice. i.e., for the convergence guarantees to hold, MAIQL must see ``approximately'' infinitely many of all state-action pairs before updating $\lambda$ once. This can be difficult to ensure in practice for arms with transition probabilities near $0$ or $1$, and for problems where the budget is small, since many more samples of (s,a) pairs with cheap actions will be collected than ones with expensive actions.

\section{Algorithm: LPQL}
\label{sec:LPQL}

In this section, we provide a more fundamental approach by studying the problem of minimizing $J(\cdot,\lambda)$~(Equation~\ref{eq:decoupled_value_func}) over $\lambda$. By minimizing this value, we aim to compute a tight bound on Eq.~\ref{eq:combined_value_function}, the value function of the original, non-relaxed problem, then follow the policies implied by the bound, i.e., the Lagrange policy. However, computing $J(\cdot,\cdot)$ requires the $Q^i(s,a,\lambda)$ values which in turn require the knowledge of transition probabilities (as shown in Equation~\ref{eq:arm_value_func_lagrange}). In absence of the knowledge of transition probabilities, we propose a method, called \textbf{L}agrange \textbf{P}olicy \textbf{Q}-\textbf{L}earning (LPQL). This method learns a representation of $Q^i(s,a_j,\lambda)$ by using samples obtained from the environment via a mechanism similar to MAIQL. However, rather than estimating $Q^i(s,a_j,\lambda)$ with the purpose of estimating some downstream value of $\lambda$ (i.e., indexes), now the goal is to estimate the entire curve $Q^i(s,a_j,\lambda)$ with respect to $\lambda$. It is straightforward to show that $Q^i(s,a_j,\lambda)$ is convex decreasing in $\lambda$ \cite{hawkins2003langrangian}, meaning that once we have a representation of $Q^i(s,a_j,\lambda)$, minimizing $J(\cdot,\cdot)$ simply corresponds to a one-dimensional convex optimization problem that can be solved extremely quickly.

In addition to its speed, this approach is desirable because it is designed for RMAB instances without specific structures, i.e., LPQL accommodates arbitrary action costs and needs no assumption on indexability. 
It does so by computing the Lagrange policy, which is asymptotically optimal for binary-action RMABs regardless of indexability \cite{weber1990index}, and works extremely well in practice for multi-action settings \cite{killian2021multiAction}. LPQL enjoys these benefits, and further, is designed to work on a single learning timescale, making its convergence faster and more stable than MAIQL. 


In the offline setting, $J(\cdot,\cdot)$ can be minimized by solving this linear program (LP), which can be derived directly from Eq.~\ref{eq:decoupled_value_func} \cite{hawkins2003langrangian}:
\begin{equation}
\begin{aligned}\label{eq:the_lp}
    &\min_{\lambda}J(\bm{s}, \lambda) = \min_{V^i(s^i, \lambda), \lambda} \frac{\lambda B}{1-\beta} + \sum_{i=0}^{N-1}\mu^i(s^i) V^i(s^i, \lambda) \\
    &\text{s.t. }V^i(s^i, \lambda) \ge r^i(s^i) - \lambda c_j + \beta\sum_{s^{i\prime}}T(s^i, a_j^i, s^{i\prime}) V^i(s^{i\prime},\lambda) \\
    & \forall i \in \{0,...,N-1\},\hspace{2mm} \forall s^i \in \mathcal{S},\hspace{2mm} \forall a_j \in \mathcal{A},\text{ and } \lambda \ge 0
\end{aligned}
\end{equation}
where $\mu^i(s^i) = 1$ if $s^i$ is the start state for arm $i$ and is $0$ otherwise and $V^i(s^i, \lambda)=\max_{a_j}\{Q^i(s^i,a_j, \lambda)\}$. 
To learn $Q^i(s^i,a_j, \lambda)$ in the offline setting, we will build a piecewise- linear convex representation of the curve by estimating its value at various points $\lambda_p$. To do this, we keep a three-dimensional vector for each arm $Q(s,a_j, \lambda)\ \in \mathbb{R}^{|\mathcal{S}|\times M\times n_{lam}}$ where $n_{lam}$ is the number of points $\lambda_p$ at which to estimate the curve. For now, we choose the set of $\lambda_p$ to be an equally spaced grid between 0 and some value $\lambda_{\max}$. Since $V^i(s,\lambda)$ is convex decreasing in $\lambda$, the largest possible value of $\lambda$ that could be a minimizer of $J(\cdot,\cdot)$ is the $\lambda$ where $\frac{dQ^i(s,a_j,\lambda)}{d\lambda}=0$. Killian et al.~\cite{killian2021multiAction} show that this value is no greater than $\frac{\max\{r\}}{\min\{\mathcal{C}\} (1-\beta)}$, so this will serve as $\lambda_{\max}$ unless otherwise specified.

On each round, an $(s,a_j,r,s^\prime)$ tuple is sampled for each arm. We store estimates of Q for each state, action, and $\lambda_p$ value, requiring $\mathcal{O}(N n_{lam} |\mathcal{S}|M)$ memory. The update rule for $Q(s, a_j, \lambda_p)$ is:
\begin{align}
    Q^{t+1}&(s, a_j, \lambda_p) = Q^t(s, a_j, \lambda_p) + \alpha(\nu(s, a_j, n))\ast \nonumber\\
    &\big [ [r(s) - \lambda_p c_j + \beta\max_{a_j^\prime\in \mathcal{A}} Q^t(s^\prime, a_j^\prime, \lambda_p)] - Q^t(s, a_j, \lambda_p) \big ]
\end{align}
Where $\beta$ is the discount factor. Each round, we sample a $(s,a_j,r,s^\prime)$ tuple per arm, and for each arm loop to update $Q^{t+1}(s, a_j, \lambda_p)$ $\forall p$. As in MAIQL, this update can be parallelized but requires $\mathcal{O}(N n_{lam})$ time if computed serially. To choose a policy each round, we compute the minimum of Eq.~\ref{eq:decoupled_value_func} 
by finding the point at which increasing $\lambda_p$ (stepping from 0, $\frac{\lambda_{\max}}{n_{lam}},\dots,\lambda_{\max}$) results in zero or positive change in objective value, as computed via our estimates $Q(s, a_j, \lambda_p)$, taking $\mathcal{O}(n_{lam})$ time. As our estimates $Q(s, a, \lambda_p)$ converge, we approximate points exactly on the true $Q(s, a_j, \lambda)$ curve. Even at convergence, there will be some small approximation error in the slope of the line that will manifest as error in the objective value, but in the next subsection, we show that the approximation error can be made arbitrarily small as $n_{lam}$ increases.

Once the minimizing value of $\lambda$ ($\lambda_{min}$) is found, we follow the knapsack from \cite{killian2021multiAction} to select actions, i.e., we input $Q(s, a_j, \lambda_{min})$ as values in a knapsack where the costs are the corresponding $c_j$ and the budget is $B$. We then use the Gurobi optimizer software \cite{gurobi} to solve the knapsack, then carry out the policy in accordance with the selected actions, taking $\mathcal{O}(NMB)$ time in total \cite{killian2021multiAction}. Pseudocode for LPQL is given in the appendix.

\subsection{Theoretical Guarantees}

We establish that, given a $\lambda_{\max}$, a higher $n_{lam}$ results in a better approximation of the upper bound of the policy return, given in Eq.~\ref{eq:decoupled_value_func}. We show that, given a state profile $\bm{s}=\{s^1,\ldots,s^{N}\}$, the asymptotic values of $V^i(s^i,\lambda)$ obtained at equally spaced discrete set of $\lambda$ values (over-)approximates Equation~\ref{eq:the_lp}. The smaller the intervals are, the closer is the approximated value of $J(\bm{s}, \lambda)$ at all $\lambda$ points that are not at the interval points. Before stating the theorem formally, we define the \textit{Chordal Slope} Lemma. For ease of representation, we drop the notations $\bm{s}$ and $s^i$ from functions $V()$ and $J()$ and also remove the superscript $i$.

\begin{lemma}[The Chordal Slope Lemma~\cite{notes}] Let $F$ be a convex function on $(a, b)$. If $x_1 < x < x_2$ are in $(a,b)$, then for points $P_1 = (x_1, F(x_1))$, $P = (x, F(x))$, and $P2 = (x_2, F(x_2))$, the slope of the straight line $P_1P$ is less than or equal to the slope of the straight line $P1P2$.\label{lemma:slope} 
\end{lemma}

\begin{theorem}
Let $V'(\cdot)$ be a convex piecewise-linear function over equally spaced intervals ($\Lambda:=\{0, x, 2x, 3x,\ldots\}$) that approximates the convex decreasing function $V(\lambda)$, such that \[V(\lambda) = V'(\lambda) \mbox{ for all }\lambda\in\Lambda.\]
If values $V(\cdot)$ are replaced by values $V'(\cdot)$, then  $J(\cdot)$ (Equation~\ref{eq:the_lp}) is better approximated when the interval length $x$ is small.  
\end{theorem}

\begin{proof}
$V(\cdot)$ are convex functions of $\lambda$ which implies that the function $J(\lambda)$, the sum of convex functions, is also a convex function of $\lambda$.  Let us assume that the convex decreasing function $V(\lambda)$ is approximated by a convex continuous piecewise-linear function $V'(\lambda)$, over equally spaced values, taken from the set $\Lambda:=\{0, x, 2x, 3x,\ldots\}$, such that $V(\lambda) = V'(\lambda)$ for all $\lambda\in \Lambda$. Thus, using $V'(\cdot)$ values instead of $V(\cdot)$ values, we obtain an approximation $J'(\cdot)$ of the convex function $J(\lambda)$. The function  $J'(\lambda)$ is a convex function with $J(\lambda)=J'(\lambda)$ for all $\lambda\in \Lambda$. 

Now, let us assume two different values of $x$, say $x_1$ and $x_2$, where $x_1<x_2$. The corresponding sets are $\Lambda_1:=\{0, x_1, 2x_1,\ldots\}$ and $\Lambda_2:=\{0, x_2, 2x_2,\ldots\}$. Considering $\Lambda_1$, and two points $\lambda_0\geq 0$ and $\lambda_1=\lambda_0+x_1$, $J'(\cdot)$ is over approximated by a straight line $\bar{J}(\cdot)$ that connects $(\lambda_0, J(\lambda_0))$ and $(\lambda_1, J(\lambda_1))$. The equation for the line is given by: 
\begin{eqnarray}
    \bar{J}_{\lambda_0, x_1, \lambda_1}(\lambda) = J(\lambda_0) + \frac{\lambda - \lambda_0}{x_1}(J(\lambda_1)-J(\lambda_0)) \mbox{ $\forall$ $\lambda_0\leq\lambda\leq \lambda_1$}. \label{eq:delta1}
\end{eqnarray}
Similarly, considering $\Lambda_2$, the point $\lambda_0$, and point $\lambda_2=\lambda_0+x_2$ (where $x_1<x_2$), $J'(\cdot)$ can be over approximated by a straight line $\bar{J}(\cdot)$ that connects $(\lambda_0, J(\lambda_0))$ and $(\lambda_2, J(\lambda_2))$. Thus, for any value of $\lambda\in[\lambda_0,\lambda_2]$, the difference $J'(\lambda) - J(\lambda)$ is given by:
\begin{eqnarray}
   \bar{J}_{\lambda_0, x_2, \lambda_2}(\lambda) = J(\lambda_0) + \frac{\lambda - \lambda_0}{x_2}(J(\lambda_2)-J(\lambda_0)) \mbox{ $\forall$ $\lambda_0\leq\lambda\leq \lambda_2$}. \label{eq:delta2}
\end{eqnarray}

For a given $\lambda\in[\lambda_0, \lambda_1]$, the difference between the approximation obtained by Equation~\ref{eq:delta2} and \ref{eq:delta1} is:
\begin{eqnarray}
&&  (\lambda-\lambda_0)\left(\frac{J(\lambda_2)-J(\lambda_0)}{x_2} -\frac{J(\lambda_1)-J(\lambda_0)}{x_1}\right) \nonumber \\
    && \geq   0 \quad\quad\qquad\qquad (\because \lambda\geq \lambda_0 \mbox{ and } Lemma~\ref{lemma:slope}) \label{eq:diff}
\end{eqnarray}

Thus, smaller the length of each interval, the corresponding surrogate $V'(\cdot)$ values can be used to obtain a better approximation of $J(\cdot)$ values.
\end{proof}

\subsection{Extending LPQL Update Technique to Approximate MAIQL}
The same tactic of approximating $Q(s, a_j, \lambda_p)$ can be used to create an \textbf{approximate version of MAIQL (MAIQL-Aprx)} that learns on a single timescale and is thus more sample efficient and stable. The algorithm follows much in the same way as LPQL, except that $Q(s, a, \lambda_p)$ are not used to minimize the LP. Instead, for each arm on each round, we compute the multi-action index for a given $(s, a_j)$ by finding the $\argmin_{\lambda_p}|Q(s, a_j, \lambda_p) - Q(s, a_{j-1}, \lambda_p)|$. We then choose actions according to the same greedy policy as MAIQL. We can show with the same logic as the LPQL approximation proof that with a large enough $n_{lam}$, the indexes can be approximated to an arbitrary precision. We investigate whether, due to its single timescale nature, this algorithm will have improved sample efficiency and convergence behavior compared to standard MAIQL.

\section{Experimental Results}
In this section, we compare our algorithms against both learning baselines (\textbf{WIBQL} (\citet{avrachenkov2020whittle}) and \textbf{QL-$\bm{\lambda}$=0}), and offline baselines (\textbf{Oracle LP}, \textbf{Oracle $\bm{\lambda}$=0}, and \textbf{Oracle-LP-Index}).

\textbf{WIBQL} is designed to learn Whittle indexes for \emph{binary}-action RMABs, but we adapt it to the multi-action setting by allowing it to plan using two actions, namely the passive action $a_0$ and a non-passive action $a_j$ ($j>0$) for the entire simulation. Clearly, this will be suboptimal in general, so we also design a stronger, multi-action baseline, \textbf{QL-$\bm{\lambda}$=0}. This uses standard Q-learning to learn state-action values for each individual arm without reasoning about future costs or the shared budget between arms (i.e., $\lambda=0)$. At each step, the actions are chosen according to the knapsack procedure of LPQL. \textbf{Oracle $\bm{\lambda}$=0} is the offline version of QL-$\lambda$=0 (i.e., it knows the transition probabilities). \textbf{Oracle LP} is the offline version of LPQL that solves Eq.\ref{eq:the_lp} using an LP solver, then follows the same knapsack procedure as LPQL. \textbf{Oracle-LP-Index} is an offline version of MAIQL that computes the multi-action indexes using an LP (see appendix). Since the oracles are computationally expensive, they are run for 1000 timesteps to allow their returns to converge, then are extrapolated.

All algorithms follow an $\epsilon$-greedy paradigm for exploration where $\epsilon$ decays each round according to $\epsilon_0/\ceil*{\frac{t}{D}}$ where $\epsilon_0$ and $D$ are constants. All algorithms were implemented in Python 3.7.4 and LPs were solved using Gurobi version 9.0.3 via the gurobipy interface \cite{gurobi}. All results are presented as the average (solid line) and interquartile range (shaded region) over 20 independently seeded simulations and were executed on a cluster running CentOS with Intel(R) Xeon(R) CPU E5-2683 v4 @ 2.1 GHz with 4GB of RAM.


\subsection{Two Process Types}

\begin{figure}[h]
\includegraphics[width=\columnwidth]{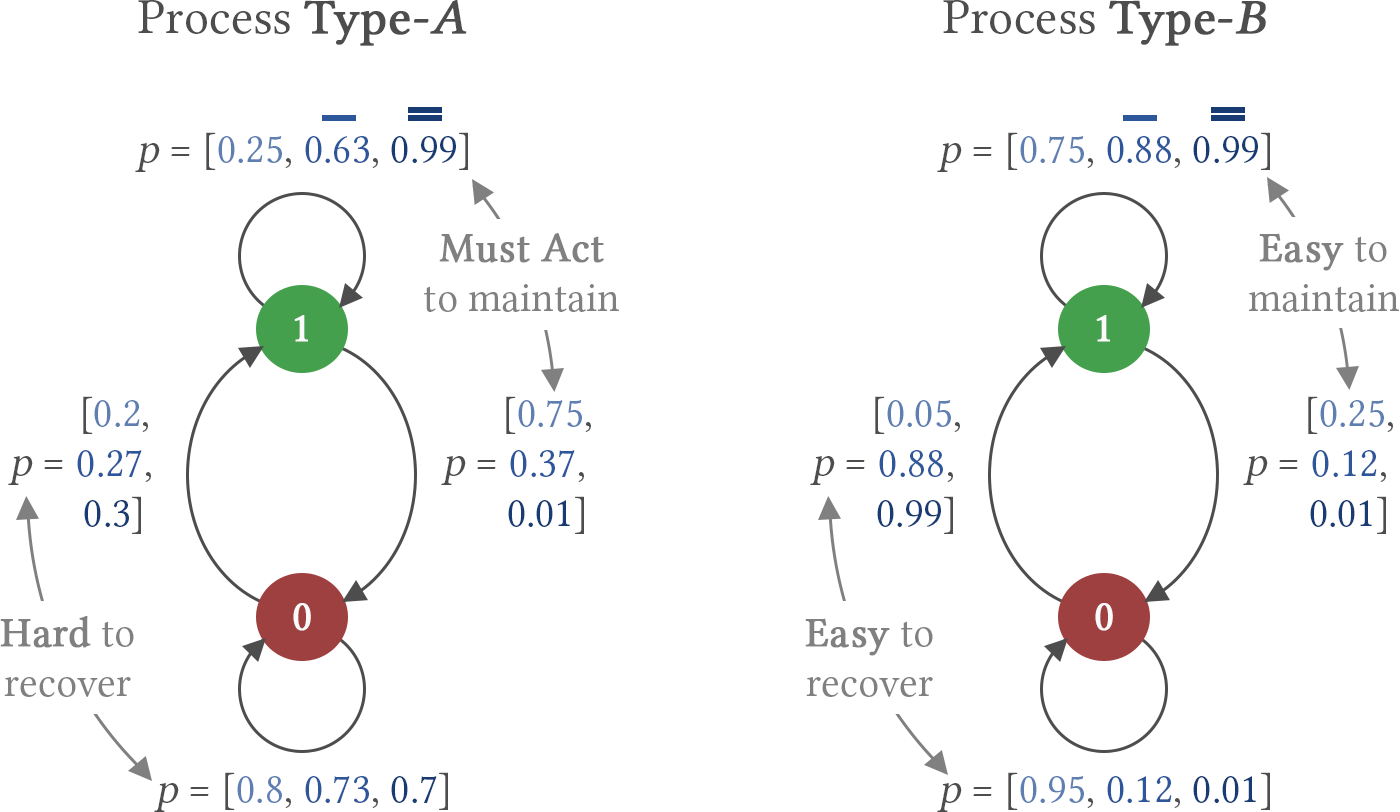}
\centering
\caption{Two Process domain. Type-A arms need constant actions to stay in the good state (reward 1), whereas Type-B arms stay in the good state for many rounds after an action.}
\label{fig:toy_domain_2}
\end{figure}

\begin{figure}[h]
\includegraphics[width=\columnwidth]{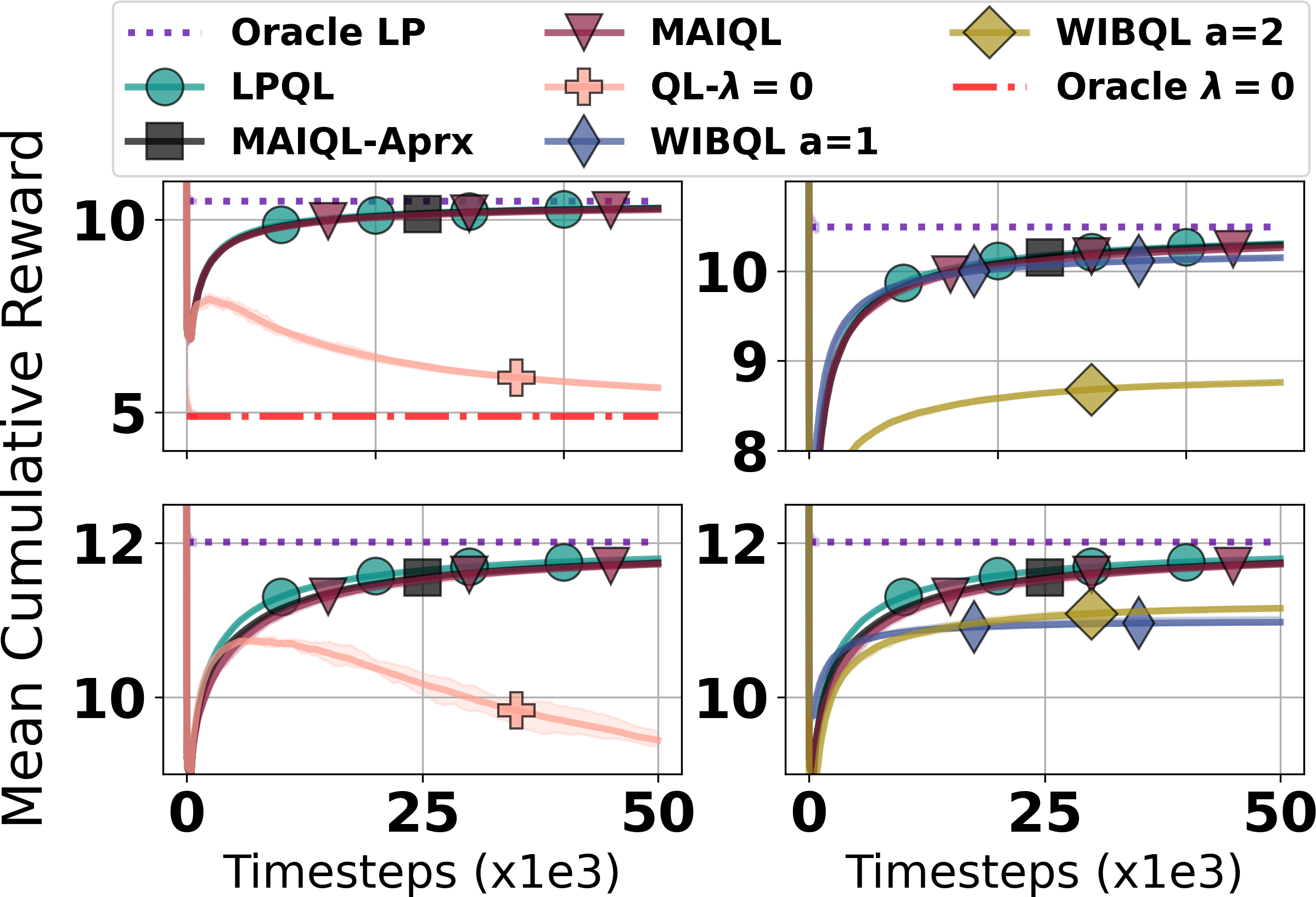}
\centering
\caption{Results from Type-A v.s.~Type-B domain with $N=16$ and  $B=4$ (top row) and $B=8$ (bottom row). Experiments in a row are the same, with different algorithms shown. Budget-agnostic learning converges to a highly suboptimal policy. Our algorithms converge to the best oracle policy, LPQL doing so the quickest. Binary-action planning underperforms except when a small budget forces the optimal policy to only use the cheapest action.}
\label{fig:toy_domain_results_2_vfnc}
\end{figure}
\label{sec:toy2}
In the first experiment, we demonstrate how failing to account for cost and budget information while learning (i.e., QL-$\lambda$=0) can lead to poorly performing policies. The setting has two types of processes (arms), as in Fig.~\ref{fig:toy_domain_2}. Each has 3 actions, with costs 0, 1, and 2. Both arms have a good and bad state that gather 1 and 0 reward, respectively. The \textbf{Type-A} arm must be acted on every round while in the good state to stay there. However, in the bad state it is difficult to recover. This leads QL-$\lambda$=0 to learn that $Q(1,a_{j>0},\lambda=0) - Q(1,a_0,\lambda=0)$ is large, i.e., acting in the good state is important for Type-A arms. Conversely, the \textbf{Type-B} arm will tend to stay in the good state even when not acted on, and when in the bad state, it can be easily recovered with any action. This leads QL-$\lambda$=0 to learn that $Q(1,a_{j>0},\lambda=0) - Q(1,a_0,\lambda=0)$ is small. Thus QL-$\lambda$=0 will prefer to act on Type-A arms. However, if the number of Type-B arms is larger than the available budget, it is clearly better to spend the budget acting on Type-B arms since the action ``goes farther'', i.e., they may spend several rounds in the good state following only a single action, v.s.~Type-A arms which are likely to only spend one round in the good state per action. Our budget-aware learning algorithms learn this tradeoff to converge to well-performing policies that greatly outperform cost-unaware planning.

We report the mean cumulative reward of each algorithm, i.e., its cumulative reward divided by the current timestep, averaged over all seeds. Fig.~\ref{fig:toy_domain_results_2_vfnc} shows the results with $N=16$, 25\% of arms as Type-A and 75\% Type-B, over $50000$ timesteps. The top and bottom rows use $B={4}$ and $B={8}$, respectively. For ease of visual representation, each column shows different combinations of algorithms -- please note that the y-axis scales for each plot may be different. Fig.~\ref{fig:toy_domain_results_vary_n} shows results for the same arm type split and simulation length with $B=8$, varying $N \in [16, 32, 48]$ (top to bottom). Parameter settings for each algorithm are included in the appendix. We see that each of our algorithms beat the baselines and converge in the limit to the Lagrange policy -- equivalent to the multi-action index policy in this case -- with the single-timescale algorithms converging quickest. Since the rewards obtained using Oracle-LP-Index coincide with Oracle LP, we do not plot the results for Oracle-LP-Index. Further, the plots demonstrate that the WIBQL algorithms underperform in general, except in cases where budgets are so small that the optimal policy effectively becomes binary-action (e.g., Fig~\ref{fig:toy_domain_results_2_vfnc} top right; $B=4$). In the remaining experiments, WIBQL is similarly dominated and so is omitted for visual clarity. In both figures, interestingly, QL-$\lambda$=0 performs well at first while $\epsilon$ is large, suggesting that a random policy would outperform the $\lambda=0$ policy. However, it eventually converges to Oracle-$\lambda$=0 as expected. 

\begin{figure}[h]
\includegraphics[width=0.8\columnwidth]{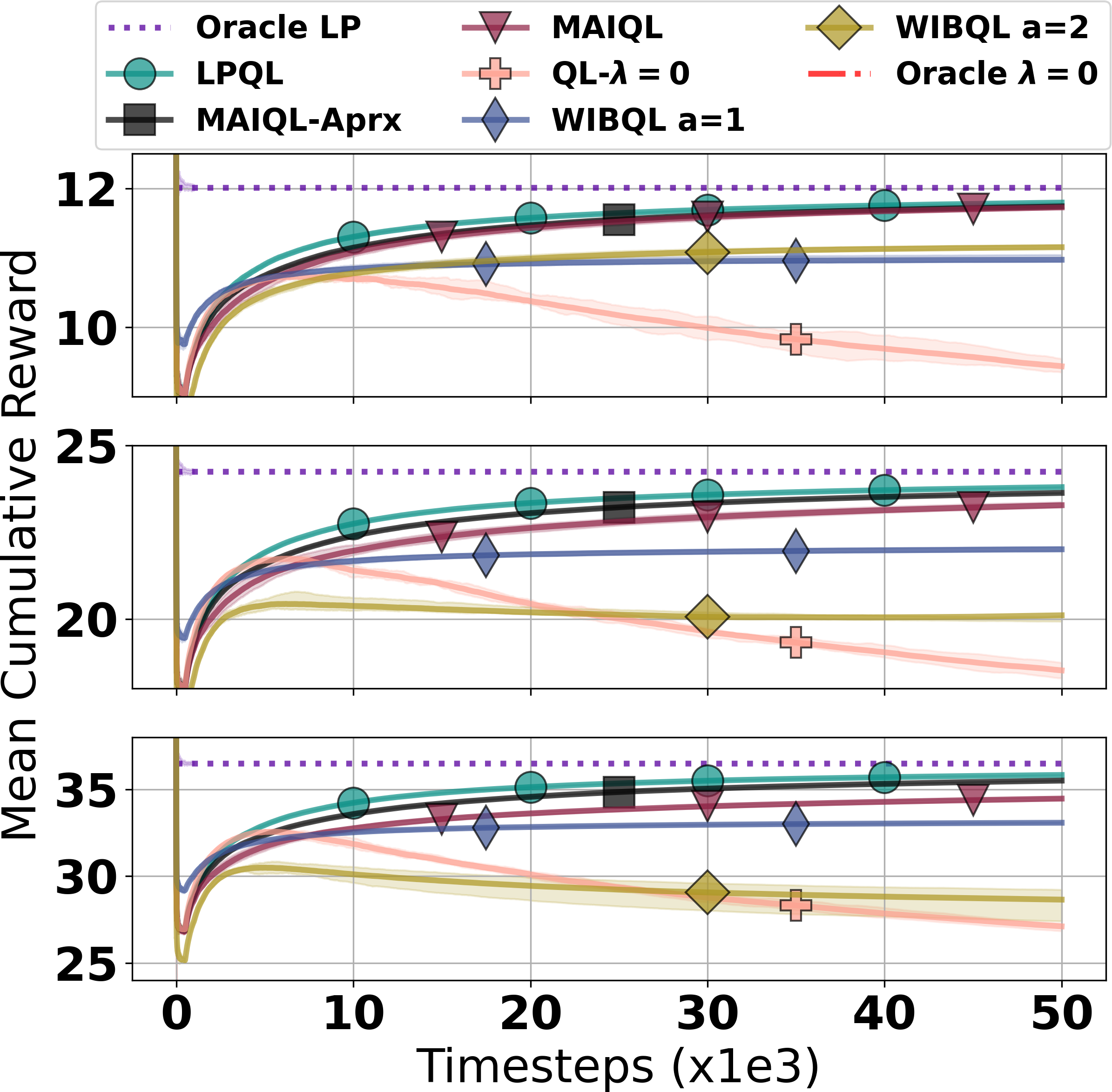}
\centering
\caption{Results from the Two Process domain with $B=8$ and  $N\in[16, 32, 48]$ (top to bottom). Budget-agnostic converges to highly suboptimal policies, while our algorithms converge to the best oracle policy, with single-timescale versions doing so the quickest. Binary action planning underperforms with the $a=2$ adaptation deteriorating as the budget becomes more constrained. Oracle $\lambda=0$ (not shown) is dominated by all lines.}
\label{fig:toy_domain_results_vary_n}
\end{figure}


\subsection{Random Matrices}
\label{sec:random_experiment}
The second experimental setting demonstrates LPQL's superior generality over index policies and its robustness to increases in the number of actions and variations in cost structure. In this setting, all transition probabilities, rewards, and costs are sampled uniformly at random, ensuring with high probability that the submodular action effect structure required for MAIQL's good performance will not exist. What remains to investigate is whether LPQL will be able to learn better policies than MAIQL in such a setting. Specifically, rewards for each state on each arm are sampled uniformly from $[0,1]$, with $|\mathcal{S}|=5$. Action costs are sampled uniformly from $[0,1]^{|\mathcal{A}|}$, then we apply a cumulative sum to ensure that costs are increasing (but $c_0$ is set to 0).  Fig.~\ref{fig:random_domain_results} shows results for $N=16$ and $B=N|\mathcal{A}|/2$ as $|\mathcal{A}|$ varies in $[2, 5, 10]$ (top to bottom) over $50000$ timesteps. Note that $B$ scales with $|\mathcal{A}|$ to ensure that optimal policies will include the additional action types, since the costs of the additional action types also scale with $|\mathcal{A}|$. Rewards are shown as a moving average with a windows size of 100, which gives a clearer view of between-seed variance than the cumulative view. Fig.~\ref{fig:random_domain_results} shows that not only is LPQL able to learn much better policies than MAIQL and MAIQL-Aprx, which themselves converge to their oracle upper bound (Oracle-LP-Index), it does so with convergence behavior that is robust to increases in the number of actions, achieving near-optimal average returns at around 10k steps in each setting. Parameter settings for the different algorithms are again included in the appendix.

\begin{figure}[h]
\includegraphics[ width=.8\columnwidth]{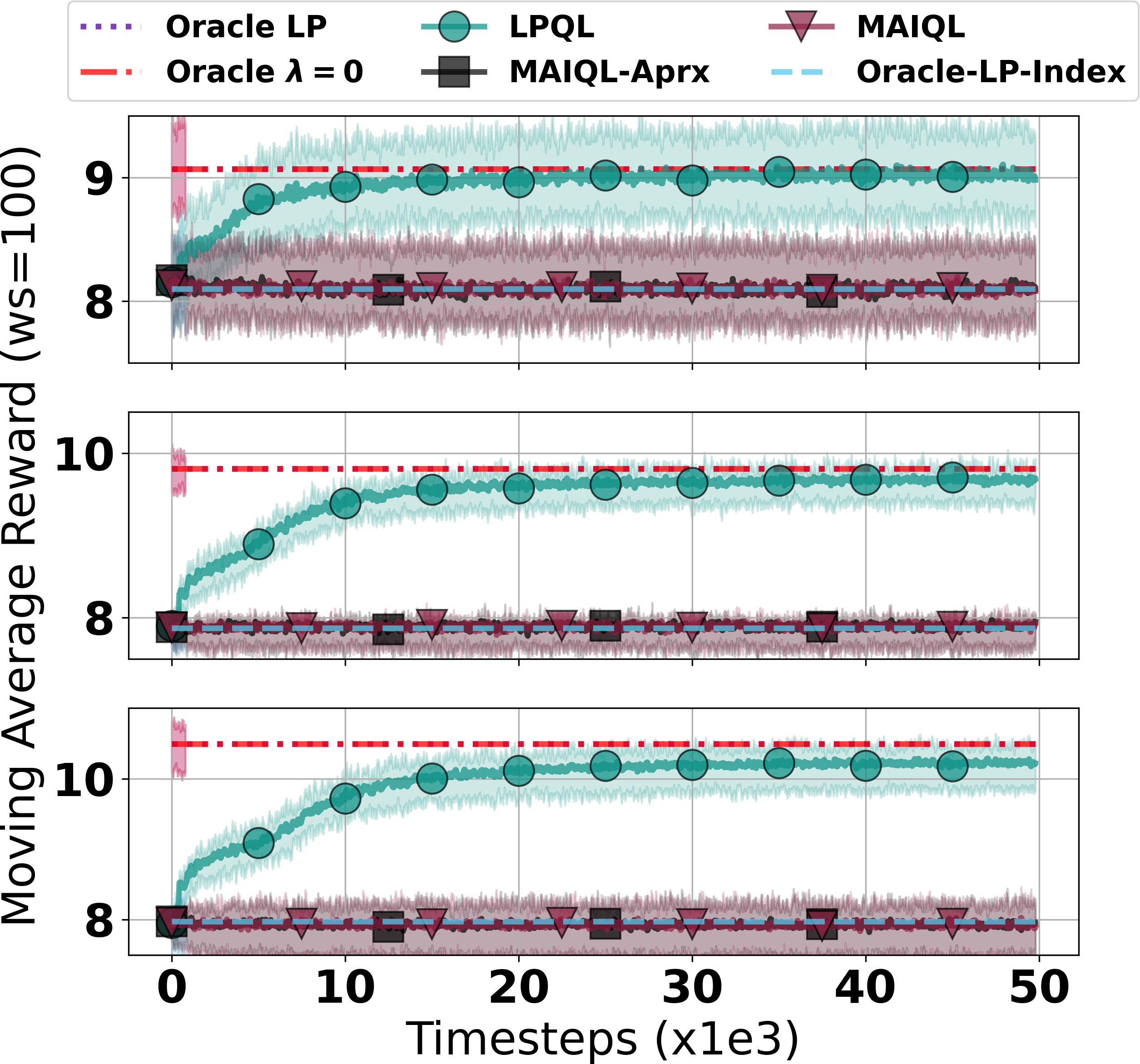}
\centering
\caption{Moving average rewards from the random domain, for $|\mathcal{A}| \in [2, 5, 10]$ (top to bottom) using a window size (ws) of 100. Oracle LP and Oracle $\lambda=0$ perform the same, as do MAIQL and MAIQL-Aprx. Oracle-LP-Index computes the index solution offline, demonstrating that MAIQL(-Aprx) are converging correctly, but the index policy performs poorly. LPQL converges quickly even as $|\mathcal{A}|$ increases.}
\label{fig:random_domain_results}
\end{figure}




\subsection{Medication Adherence}
Finally, we run an experiment using data derived in-part from a real medication adherence domain \cite{killian2019learning}. The data contains daily 0-1 records of adherence from which transition probabilities can be estimated, assuming a corresponding 0-or-1 state (partial state history can also be accommodated). However, the data contains no records of actions and so must be simulated. In this experiment, we simulate actions that assume a natural ``diminishing returns'' structure in accordance with the assumption in section \ref{sec:MAIQL}. One drawback is this estimation procedure creates uniform action effects across arms in expectation, i.e., a single ``mode''. However, in the real world we expect there to be multiple modes, representing patients' diverse counseling needs and response rates to various intervention types. To obtain multiple modes in a simple and interpretable way, we sample 25\% of arms as Type-A arms from section \ref{sec:toy2}, since they also have a binary state structure and are easily extended to accommodate partial state history. More details are given in the appendix. Note that, similar to Section~\ref{sec:toy2}, Oracle LP coincides with Oracle-LP-Index and hence, we do not plot the results for Oracle-LP-Index separately. Fig.~\ref{fig:schematic} visualizes this domain.

Fig.~\ref{fig:tb_results} shows the results for the medication adherence domain with history lengths of 2, 3, and 4 (top to bottom), $N=16$, $B=4$, and 3 actions of cost 0, 1, and 2, over $100000$ timesteps. Please see the appendix for parameter settings. This demonstrates concretely that learning on a single timescale (LPQL and MAIQL-Aprx) clearly improves speed of convergence, and this becomes more pronounced as the size of the state space increases. To understand why, we analyzed the estimated transition matrices and found that many patients had values near 0 or 1. This makes it very rare to encounter certain states, making it difficult to obtain sufficient numbers of samples across all state action pairs for MAIQL's assumptions to hold, impeding its learning.

\section{Conclusion}
To the best of our knowledge, we are the first to provide algorithms for learning Multi-action RMABs in an online setting. We show that by following the traditional approaches to RMAB problems, i.e., seeking index policies in domains with structural assumptions, MAIQL is guaranteed to converge to the optimal solution as $t \to \infty$. However, it is not efficient, due to its two-timescale structure, and is limited in scope, due to its indexability assumption. We solve these challenges by going back to the fundamentals of RMABs to develop LPQL which works well regardless of the problem structure, and outperforms all other baselines in terms of both convergence rate and obtained reward. 
Towards a real-world RMAB deployment, our models would apply to settings that allow many repeat interactions over a long horizon, e.g., life-long medication adherence regimens \cite{cote2018web}. However, since our algorithms require thousands of samples to learn, more work is needed to apply to many settings which may have short horizons. Still, this work lays a methodological and theoretical foundation for future work in online multi-action RMABs, a crucial step toward their real-world deployment.

\begin{figure}[h!]
\includegraphics[width=0.78\columnwidth]{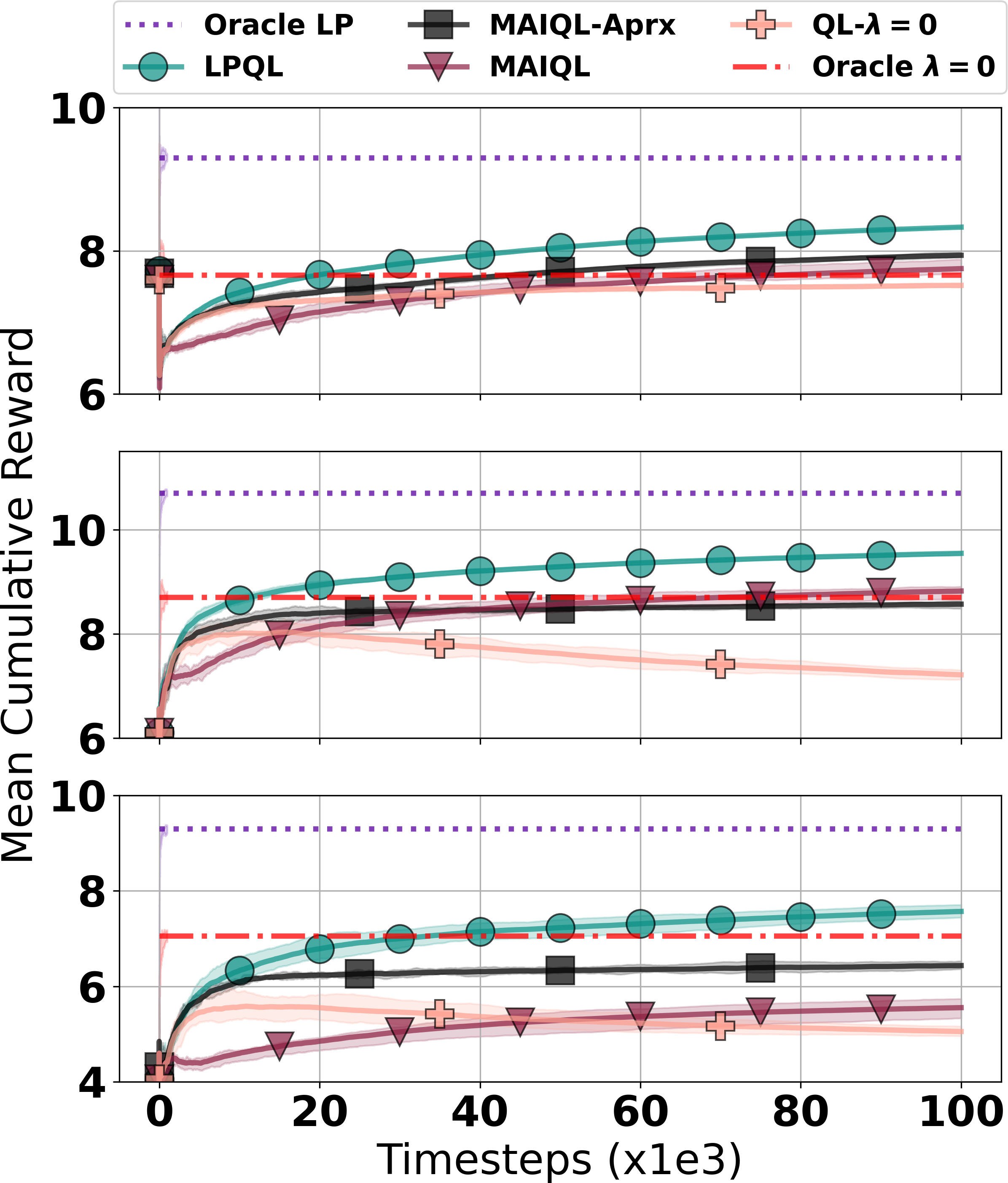}
\centering
\caption{Mean cumulative reward on medication adherence domain with 16 patients, $B=4$ and history length of 2, 3, and 4 (top to bottom). LPQL is the fastest to converge and converges to the best policies across all history lengths. MAIQL is slower to learn but does so eventually, where its approximate variant that learns on a single-timescale is more stable as the state size increases.}
\label{fig:tb_results}
\end{figure}

\begin{acks}
This work was supported in part by the Army Research Office by Multidisciplinary University Research Initiative (MURI) grant number W911NF1810208. J.A.K. was supported by an NSF Graduate Research Fellowship under grant DGE1745303. A.B. was supported by the Harvard Center for Research on Computation and Society.
\end{acks}

\bibliographystyle{ACM-Reference-Format}
\bibliography{3_bibliography.bib}


\begin{thebibliography}{32}


\ifx \showCODEN    \undefined \def \showCODEN     #1{\unskip}     \fi
\ifx \showDOI      \undefined \def \showDOI       #1{#1}\fi
\ifx \showISBNx    \undefined \def \showISBNx     #1{\unskip}     \fi
\ifx \showISBNxiii \undefined \def \showISBNxiii  #1{\unskip}     \fi
\ifx \showISSN     \undefined \def \showISSN      #1{\unskip}     \fi
\ifx \showLCCN     \undefined \def \showLCCN      #1{\unskip}     \fi
\ifx \shownote     \undefined \def \shownote      #1{#1}          \fi
\ifx \showarticletitle \undefined \def \showarticletitle #1{#1}   \fi
\ifx \showURL      \undefined \def \showURL       {\relax}        \fi
\providecommand\bibfield[2]{#2}
\providecommand\bibinfo[2]{#2}
\providecommand\natexlab[1]{#1}
\providecommand\showeprint[2][]{arXiv:#2}

\bibitem[\protect\citeauthoryear{Abounadi, Bertsekas, and Borkar}{Abounadi
  et~al\mbox{.}}{2001}]%
        {abounadi2001learning}
\bibfield{author}{\bibinfo{person}{Jinane Abounadi}, \bibinfo{person}{Dimitrib
  Bertsekas}, {and} \bibinfo{person}{Vivek~S Borkar}.}
  \bibinfo{year}{2001}\natexlab{}.
\newblock \showarticletitle{Learning algorithms for Markov decision processes
  with average cost}.
\newblock \bibinfo{journal}{\emph{SIAM J. Control Optim.}}
  \bibinfo{volume}{40}, \bibinfo{number}{3} (\bibinfo{year}{2001}),
  \bibinfo{pages}{681--698}.
\newblock


\bibitem[\protect\citeauthoryear{Adelman and Mersereau}{Adelman and
  Mersereau}{2008}]%
        {adelman2008relaxations}
\bibfield{author}{\bibinfo{person}{Daniel Adelman} {and}
  \bibinfo{person}{Adam~J Mersereau}.} \bibinfo{year}{2008}\natexlab{}.
\newblock \showarticletitle{Relaxations of weakly coupled stochastic dynamic
  programs}.
\newblock \bibinfo{journal}{\emph{Oper. Res.}} \bibinfo{volume}{56},
  \bibinfo{number}{3} (\bibinfo{year}{2008}), \bibinfo{pages}{712--727}.
\newblock


\bibitem[\protect\citeauthoryear{Avrachenkov and Borkar}{Avrachenkov and
  Borkar}{2020}]%
        {avrachenkov2020whittle}
\bibfield{author}{\bibinfo{person}{Konstantin Avrachenkov} {and}
  \bibinfo{person}{Vivek~S Borkar}.} \bibinfo{year}{2020}\natexlab{}.
\newblock \showarticletitle{Whittle index based q-learning for restless bandits
  with average reward}.
\newblock \bibinfo{journal}{\emph{arXiv preprint arXiv:2004.14427}}
  (\bibinfo{year}{2020}).
\newblock


\bibitem[\protect\citeauthoryear{Bhattacharya}{Bhattacharya}{2018}]%
        {bhattacharya2018restless}
\bibfield{author}{\bibinfo{person}{Biswarup Bhattacharya}.}
  \bibinfo{year}{2018}\natexlab{}.
\newblock \showarticletitle{Restless bandits visiting villages: A preliminary
  study on distributing public health services}. In
  \bibinfo{booktitle}{\emph{Proceedings of the 1st ACM SIGCAS Conference on
  Computing and Sustainable Societies}}. \bibinfo{pages}{1--8}.
\newblock


\bibitem[\protect\citeauthoryear{Biswas, Aggarwal, Varakantham, and
  Tambe}{Biswas et~al\mbox{.}}{2021a}]%
        {biswas2021learn}
\bibfield{author}{\bibinfo{person}{Arpita Biswas}, \bibinfo{person}{Gaurav
  Aggarwal}, \bibinfo{person}{Pradeep Varakantham}, {and}
  \bibinfo{person}{Milind Tambe}.} \bibinfo{year}{2021}\natexlab{a}.
\newblock \showarticletitle{Learn to Intervene: An Adaptive Learning Policy for
  Restless Bandits in Application to Preventive Healthcare}. In
  \bibinfo{booktitle}{\emph{Proceedings of the 30th International Joint
  Conference on Artificial Intelligence}}.
\newblock


\bibitem[\protect\citeauthoryear{Biswas, Aggarwal, Varakantham, and
  Tambe}{Biswas et~al\mbox{.}}{2021b}]%
        {biswas2021learning}
\bibfield{author}{\bibinfo{person}{Arpita Biswas}, \bibinfo{person}{Gaurav
  Aggarwal}, \bibinfo{person}{Pradeep Varakantham}, {and}
  \bibinfo{person}{Milind Tambe}.} \bibinfo{year}{2021}\natexlab{b}.
\newblock \showarticletitle{Learning Index Policies for Restless Bandits with
  Application to Maternal Healthcare}. In \bibinfo{booktitle}{\emph{Proceedings
  of the 20th International Conference on Autonomous Agents and MultiAgent
  Systems}}. \bibinfo{pages}{1467--1468}.
\newblock


\bibitem[\protect\citeauthoryear{C{\^o}t{\'e}, Fortin, Auger, Rouleau, Dubois,
  Boudreau, Vaillant, and G{\'e}linas-Lemay}{C{\^o}t{\'e}
  et~al\mbox{.}}{2018}]%
        {cote2018web}
\bibfield{author}{\bibinfo{person}{Jos{\'e} C{\^o}t{\'e}},
  \bibinfo{person}{Marie-Chantal Fortin}, \bibinfo{person}{Patricia Auger},
  \bibinfo{person}{Genevi{\`e}ve Rouleau}, \bibinfo{person}{Sylvie Dubois},
  \bibinfo{person}{Nathalie Boudreau}, \bibinfo{person}{Isabelle Vaillant},
  {and} \bibinfo{person}{{\'E}lisabeth G{\'e}linas-Lemay}.}
  \bibinfo{year}{2018}\natexlab{}.
\newblock \showarticletitle{Web-based tailored intervention to support optimal
  medication adherence among kidney transplant recipients: pilot parallel-group
  randomized controlled trial}.
\newblock \bibinfo{journal}{\emph{JMIR Form. Res.}} \bibinfo{volume}{2},
  \bibinfo{number}{2} (\bibinfo{year}{2018}), \bibinfo{pages}{e14}.
\newblock


\bibitem[\protect\citeauthoryear{Fu, Nazarathy, Moka, and Taylor}{Fu
  et~al\mbox{.}}{2019}]%
        {fu2019towards}
\bibfield{author}{\bibinfo{person}{Jing Fu}, \bibinfo{person}{Yoni Nazarathy},
  \bibinfo{person}{Sarat Moka}, {and} \bibinfo{person}{Peter~G Taylor}.}
  \bibinfo{year}{2019}\natexlab{}.
\newblock \showarticletitle{Towards Q-learning the Whittle Index for Restless
  Bandits}. In \bibinfo{booktitle}{\emph{2019 Australian \& New Zealand Control
  Conference (ANZCC)}}. IEEE, \bibinfo{pages}{249--254}.
\newblock


\bibitem[\protect\citeauthoryear{Gafni and Cohen}{Gafni and Cohen}{2020}]%
        {gafni2020learning}
\bibfield{author}{\bibinfo{person}{Tomer Gafni} {and} \bibinfo{person}{Kobi
  Cohen}.} \bibinfo{year}{2020}\natexlab{}.
\newblock \showarticletitle{Learning in restless multi-armed bandits via
  adaptive arm sequencing rules}.
\newblock \bibinfo{journal}{\emph{IEEE Trans. Automat. Control}}
  (\bibinfo{year}{2020}).
\newblock


\bibitem[\protect\citeauthoryear{Gardner}{Gardner}{2017}]%
        {notes}
\bibfield{author}{\bibinfo{person}{Robert Gardner}.}
  \bibinfo{year}{2017}\natexlab{}.
\newblock \bibinfo{title}{Convex Functions}.
\newblock
  \bibinfo{howpublished}{\url{https://faculty.etsu.edu/gardnerr/5210/Beamer-Proofs/Proofs-6-6-print.pdf}}.
\newblock
\newblock
\shownote{Accessed: 2021-01-15.}


\bibitem[\protect\citeauthoryear{Glazebrook, Hodge, and Kirkbride}{Glazebrook
  et~al\mbox{.}}{2011}]%
        {glazebrook2011general}
\bibfield{author}{\bibinfo{person}{Kevin~D Glazebrook},
  \bibinfo{person}{David~J Hodge}, {and} \bibinfo{person}{Chris Kirkbride}.}
  \bibinfo{year}{2011}\natexlab{}.
\newblock \showarticletitle{General notions of indexability for queueing
  control and asset management}.
\newblock \bibinfo{journal}{\emph{Ann. Appl. Probab.}} (\bibinfo{year}{2011}),
  \bibinfo{pages}{876--907}.
\newblock


\bibitem[\protect\citeauthoryear{Glazebrook, Ruiz-Hernandez, and
  Kirkbride}{Glazebrook et~al\mbox{.}}{2006}]%
        {glazebrook2006some}
\bibfield{author}{\bibinfo{person}{Kevin~D Glazebrook}, \bibinfo{person}{D.
  Ruiz-Hernandez}, {and} \bibinfo{person}{Chris Kirkbride}.}
  \bibinfo{year}{2006}\natexlab{}.
\newblock \showarticletitle{Some indexable families of restless bandit
  problems}.
\newblock \bibinfo{journal}{\emph{Adv. Appl. Probab.}} \bibinfo{volume}{38},
  \bibinfo{number}{3} (\bibinfo{year}{2006}), \bibinfo{pages}{643--672}.
\newblock


\bibitem[\protect\citeauthoryear{Gocgun and Ghate}{Gocgun and Ghate}{2012}]%
        {gocgun2012lagrangian}
\bibfield{author}{\bibinfo{person}{Yasin Gocgun} {and} \bibinfo{person}{Archis
  Ghate}.} \bibinfo{year}{2012}\natexlab{}.
\newblock \showarticletitle{Lagrangian relaxation and constraint generation for
  allocation and advanced scheduling}.
\newblock \bibinfo{journal}{\emph{Computers \& Operations Research}}
  \bibinfo{volume}{39}, \bibinfo{number}{10} (\bibinfo{year}{2012}),
  \bibinfo{pages}{2323--2336}.
\newblock


\bibitem[\protect\citeauthoryear{Gurobi~Optimization}{Gurobi~Optimization}{2021}]%
        {gurobi}
\bibfield{author}{\bibinfo{person}{LLC Gurobi~Optimization}.}
  \bibinfo{year}{2021}\natexlab{}.
\newblock \bibinfo{title}{Gurobi Optimizer Reference Manual}.
\newblock
\newblock
\urldef\tempurl%
\url{http://www.gurobi.com}
\showURL{%
\tempurl}


\bibitem[\protect\citeauthoryear{Hawkins}{Hawkins}{2003}]%
        {hawkins2003langrangian}
\bibfield{author}{\bibinfo{person}{Jeffrey~Thomas Hawkins}.}
  \bibinfo{year}{2003}\natexlab{}.
\newblock \emph{\bibinfo{title}{A Langrangian decomposition approach to weakly
  coupled dynamic optimization problems and its applications}}.
\newblock \bibinfo{thesistype}{Ph.D. Dissertation}.
  \bibinfo{school}{Massachusetts Institute of Technology}.
\newblock


\bibitem[\protect\citeauthoryear{Hodge and Glazebrook}{Hodge and
  Glazebrook}{2015}]%
        {hodge2015asymptotic}
\bibfield{author}{\bibinfo{person}{David~J Hodge} {and}
  \bibinfo{person}{Kevin~D Glazebrook}.} \bibinfo{year}{2015}\natexlab{}.
\newblock \showarticletitle{On the asymptotic optimality of greedy index
  heuristics for multi-action restless bandits}.
\newblock \bibinfo{journal}{\emph{Adv. Appl. Probab}} \bibinfo{volume}{47},
  \bibinfo{number}{3} (\bibinfo{year}{2015}), \bibinfo{pages}{652--667}.
\newblock


\bibitem[\protect\citeauthoryear{Iannello, Simeone, and Spagnolini}{Iannello
  et~al\mbox{.}}{2012}]%
        {Ianello2012}
\bibfield{author}{\bibinfo{person}{Fabio Iannello}, \bibinfo{person}{Osvaldo
  Simeone}, {and} \bibinfo{person}{Umberto Spagnolini}.}
  \bibinfo{year}{2012}\natexlab{}.
\newblock \showarticletitle{Optimality of myopic scheduling and whittle
  indexability for energy harvesting sensors}. In
  \bibinfo{booktitle}{\emph{2012 46th Annual Conference on Information Sciences
  and Systems (CISS)}}. IEEE, \bibinfo{pages}{1--6}.
\newblock


\bibitem[\protect\citeauthoryear{Killian, Perrault, and Tambe}{Killian
  et~al\mbox{.}}{2021}]%
        {killian2021multiAction}
\bibfield{author}{\bibinfo{person}{Jackson~A Killian}, \bibinfo{person}{Andrew
  Perrault}, {and} \bibinfo{person}{Milind Tambe}.}
  \bibinfo{year}{2021}\natexlab{}.
\newblock \showarticletitle{Beyond “To Act or Not to Act”: Fast Lagrangian
  Approaches to General Multi-Action Restless Bandits}. In
  \bibinfo{booktitle}{\emph{Proceedings of the 20th International Conference on
  Autonomous Agents and Multiagent Systems}}.
\newblock


\bibitem[\protect\citeauthoryear{Killian, Wilder, Sharma, Choudhary, Dilkina,
  and Tambe}{Killian et~al\mbox{.}}{2019}]%
        {killian2019learning}
\bibfield{author}{\bibinfo{person}{Jackson~A Killian}, \bibinfo{person}{Bryan
  Wilder}, \bibinfo{person}{Amit Sharma}, \bibinfo{person}{Vinod Choudhary},
  \bibinfo{person}{Bistra Dilkina}, {and} \bibinfo{person}{Milind Tambe}.}
  \bibinfo{year}{2019}\natexlab{}.
\newblock \showarticletitle{Learning to prescribe interventions for
  tuberculosis patients using digital adherence data}. In
  \bibinfo{booktitle}{\emph{Proceedings of the 25th ACM SIGKDD International
  Conference on Knowledge Discovery \& Data Mining}}.
  \bibinfo{pages}{2430--2438}.
\newblock


\bibitem[\protect\citeauthoryear{Lakshminarayanan and
  Bhatnagar}{Lakshminarayanan and Bhatnagar}{2017}]%
        {lakshminarayanan2017stability}
\bibfield{author}{\bibinfo{person}{Chandrashekar Lakshminarayanan} {and}
  \bibinfo{person}{Shalabh Bhatnagar}.} \bibinfo{year}{2017}\natexlab{}.
\newblock \showarticletitle{A stability criterion for two timescale stochastic
  approximation schemes}.
\newblock \bibinfo{journal}{\emph{Automatica}}  \bibinfo{volume}{79}
  (\bibinfo{year}{2017}), \bibinfo{pages}{108--114}.
\newblock


\bibitem[\protect\citeauthoryear{Lee, Lavieri, and Volk}{Lee
  et~al\mbox{.}}{2019}]%
        {lee2019optimal}
\bibfield{author}{\bibinfo{person}{Elliot Lee}, \bibinfo{person}{Mariel~S
  Lavieri}, {and} \bibinfo{person}{Michael Volk}.}
  \bibinfo{year}{2019}\natexlab{}.
\newblock \showarticletitle{Optimal screening for hepatocellular carcinoma: A
  restless bandit model}.
\newblock \bibinfo{journal}{\emph{Manuf. Serv. Oper. Manag.}}
  \bibinfo{volume}{21}, \bibinfo{number}{1} (\bibinfo{year}{2019}),
  \bibinfo{pages}{198--212}.
\newblock


\bibitem[\protect\citeauthoryear{Liu and Zhao}{Liu and Zhao}{2010}]%
        {Liu2010}
\bibfield{author}{\bibinfo{person}{Keqin Liu} {and} \bibinfo{person}{Qing
  Zhao}.} \bibinfo{year}{2010}\natexlab{}.
\newblock \showarticletitle{Indexability of restless bandit problems and
  optimality of whittle index for dynamic multichannel access}.
\newblock \bibinfo{journal}{\emph{IEEE Trans. Inf. Theory}}
  \bibinfo{volume}{56}, \bibinfo{number}{11} (\bibinfo{year}{2010}),
  \bibinfo{pages}{5547--5567}.
\newblock


\bibitem[\protect\citeauthoryear{Mate, Killian, Xu, Perrault, and Tambe}{Mate
  et~al\mbox{.}}{2020}]%
        {adityamate2020collapsing}
\bibfield{author}{\bibinfo{person}{Aditya Mate}, \bibinfo{person}{Jackson~A
  Killian}, \bibinfo{person}{Haifend Xu}, \bibinfo{person}{Andrew Perrault},
  {and} \bibinfo{person}{Milind Tambe}.} \bibinfo{year}{2020}\natexlab{}.
\newblock \showarticletitle{Collapsing Bandits and Their Application to Public
  Health Interventions}. In \bibinfo{booktitle}{\emph{Neural Information
  Processing Systems, NeurIPS}}.
\newblock


\bibitem[\protect\citeauthoryear{Papadimitriou and Tsitsiklis}{Papadimitriou
  and Tsitsiklis}{1994}]%
        {papadimitriou1994complexity}
\bibfield{author}{\bibinfo{person}{Christos~H Papadimitriou} {and}
  \bibinfo{person}{John~N Tsitsiklis}.} \bibinfo{year}{1994}\natexlab{}.
\newblock \showarticletitle{The complexity of optimal queueing network
  control}. In \bibinfo{booktitle}{\emph{Proceedings of IEEE 9th Annual
  Conference on Structure in Complexity Theory}}. IEEE,
  \bibinfo{pages}{318--322}.
\newblock


\bibitem[\protect\citeauthoryear{Puterman}{Puterman}{2014}]%
        {puterman2014markov}
\bibfield{author}{\bibinfo{person}{Martin~L Puterman}.}
  \bibinfo{year}{2014}\natexlab{}.
\newblock \bibinfo{booktitle}{\emph{Markov Decision Processes: Discrete
  Stochastic Dynamic Programming}}.
\newblock \bibinfo{publisher}{John Wiley \& Sons}.
\newblock


\bibitem[\protect\citeauthoryear{Qian, Zhang, Krishnamachari, and Tambe}{Qian
  et~al\mbox{.}}{2016}]%
        {qian2016restless}
\bibfield{author}{\bibinfo{person}{Yundi Qian}, \bibinfo{person}{Chao Zhang},
  \bibinfo{person}{Bhaskar Krishnamachari}, {and} \bibinfo{person}{Milind
  Tambe}.} \bibinfo{year}{2016}\natexlab{}.
\newblock \showarticletitle{Restless poachers: Handling
  exploration-exploitation tradeoffs in security domains}. In
  \bibinfo{booktitle}{\emph{Proceedings of the 2016 International Conference on
  Autonomous Agents \& Multiagent Systems}}. \bibinfo{pages}{123--131}.
\newblock


\bibitem[\protect\citeauthoryear{Ruiz-Hern{\'a}ndez, Pinar-P{\'e}rez, and
  Delgado-G{\'o}mez}{Ruiz-Hern{\'a}ndez et~al\mbox{.}}{2020}]%
        {ruiz2020multi}
\bibfield{author}{\bibinfo{person}{Diego Ruiz-Hern{\'a}ndez},
  \bibinfo{person}{Jes{\'u}s~M Pinar-P{\'e}rez}, {and} \bibinfo{person}{David
  Delgado-G{\'o}mez}.} \bibinfo{year}{2020}\natexlab{}.
\newblock \showarticletitle{Multi-machine preventive maintenance scheduling
  with imperfect interventions: A restless bandit approach}.
\newblock \bibinfo{journal}{\emph{Comput. Oper. Res.}}  \bibinfo{volume}{119}
  (\bibinfo{year}{2020}), \bibinfo{pages}{104927}.
\newblock


\bibitem[\protect\citeauthoryear{Sombabu, Mate, Manjunath, and Moharir}{Sombabu
  et~al\mbox{.}}{2020}]%
        {Sombabu2020}
\bibfield{author}{\bibinfo{person}{Bejjipuram Sombabu}, \bibinfo{person}{Aditya
  Mate}, \bibinfo{person}{D Manjunath}, {and} \bibinfo{person}{Sharayu
  Moharir}.} \bibinfo{year}{2020}\natexlab{}.
\newblock \showarticletitle{Whittle index for AoI-aware scheduling}. In
  \bibinfo{booktitle}{\emph{IEEE International Conference on Communication
  Systems \& Networks (COSMSNETS)}}. IEEE.
\newblock


\bibitem[\protect\citeauthoryear{Watkins and Dayan}{Watkins and Dayan}{1992}]%
        {watkins1992q}
\bibfield{author}{\bibinfo{person}{Christopher~JCH Watkins} {and}
  \bibinfo{person}{Peter Dayan}.} \bibinfo{year}{1992}\natexlab{}.
\newblock \showarticletitle{Q-learning}.
\newblock \bibinfo{journal}{\emph{Machine learning}} \bibinfo{volume}{8},
  \bibinfo{number}{3-4} (\bibinfo{year}{1992}), \bibinfo{pages}{279--292}.
\newblock


\bibitem[\protect\citeauthoryear{Watkins}{Watkins}{1989}]%
        {watkins1989learning}
\bibfield{author}{\bibinfo{person}{Christopher John Cornish~Hellaby Watkins}.}
  \bibinfo{year}{1989}\natexlab{}.
\newblock \showarticletitle{Learning from delayed rewards}.
\newblock  (\bibinfo{year}{1989}).
\newblock


\bibitem[\protect\citeauthoryear{Weber and Weiss}{Weber and Weiss}{1990}]%
        {weber1990index}
\bibfield{author}{\bibinfo{person}{Richard~R Weber} {and}
  \bibinfo{person}{Gideon Weiss}.} \bibinfo{year}{1990}\natexlab{}.
\newblock \showarticletitle{On an index policy for restless bandits}.
\newblock \bibinfo{journal}{\emph{J. Appl. Probab.}} \bibinfo{volume}{27},
  \bibinfo{number}{3} (\bibinfo{year}{1990}), \bibinfo{pages}{637--648}.
\newblock


\bibitem[\protect\citeauthoryear{Whittle}{Whittle}{1988}]%
        {whittle1988restless}
\bibfield{author}{\bibinfo{person}{Peter Whittle}.}
  \bibinfo{year}{1988}\natexlab{}.
\newblock \showarticletitle{Restless bandits: Activity allocation in a changing
  world}.
\newblock \bibinfo{journal}{\emph{J. Appl. Probab.}} \bibinfo{volume}{25},
  \bibinfo{number}{A} (\bibinfo{year}{1988}), \bibinfo{pages}{287--298}.
\newblock


\end{thebibliography}

\appendix
\clearpage

\section{Proof of convergence for MAIQL}\label{sec:maiqlproof}
In this section, we provide a detailed proof of the convergence for MAIQL. We begin by stating 2 standard assumptions for establishing the convergence guarantee of Q-learning in the average-reward setting, and then add a third that's required for two time-scale convergence.

\begin{assumption}[Uni-chain Property]\label{ass:unichain}
    There exists a state $s_0$ that is reachable from any other state $s \in S$ with a positive probability under any policy.
\end{assumption}
\noindent This property formalises the notion that there aren't any `forks' in the MDP, in each of which very different outcomes could occur. This is important because, if there were a fork, the notion of `average' reward would be ill-defined as it would depend on which `fork' gets taken.

\begin{assumption}[Asynchronous Update Step-Size]\label{ass:async}
    The sequence of step-sizes $\{\alpha(t)\}$ satisfy the following properties for any $x \in (0, 1)$:
    $$\sup_t \frac{\alpha(\lfloor x t \rfloor)}{\alpha(t)} < \infty $$
    $$  \sup_{y\in [x, 1]} \left | \frac{\sum_{m=0}^{\lfloor y t \rfloor} \alpha(m)}{\sum_{m=0}^{t}\alpha(m)} - 1 \right | \to 0$$
\end{assumption}
\noindent This is a condition that is required to show that updating $Q(s, a_j)$ values one at a time with an $\epsilon$-greedy policy is equivalent to updating all the $Q(s, a_j)$ values together, in expectation.

\begin{assumption}[Relative Step-Size]\label{ass:alpha}
    The two sequences of step-sizes, $\{\alpha(t)\}$ and $\{\gamma(t)\}$, satisfy the following properties:
    \begin{alignat*}{3}
        &\text{\emph{(A)}} \; & \text{Fast Time-Scale:} \qquad
        &\sum_{t=0}^{\infty} \alpha(t) \to \infty \text{,} \qquad
        & \sum_{t=0}^{\infty} \alpha^2(t) < \infty \\ 
        & \text{\emph{(B)}} \; & \text{Slow Time-Scale:}  \qquad
        &\sum_{t=0}^{\infty} \gamma(t) \to \infty \text{,} \qquad
        & \sum_{t=0}^{\infty} \gamma^2(t) < \infty \\ 
    \end{alignat*}
    $$\text{\emph{(C)}} \quad \lim_{t \to \infty} \frac{\gamma(t)}{\alpha(t)} \to 0$$
\end{assumption}
\noindent An example of possible step sizes for which this condition is true is $\alpha(t) = \frac{1}{t}$ and $\gamma(t) = \frac{1}{t\log{t}}$. In our experiments we use $\alpha(t) = \frac{C}{\lceil \frac{t}{D} \rceil}$, and $\gamma(t) = \frac{C^\prime}{1 + \lceil \frac{t\log(t)}{D} \rceil}$.

We then detail the proof for Theorem \ref{thm:MAIQL} below. This proof involves mapping the discrete Q and $\lambda$ updates from the MAIQL algorithm (Section \ref{sec:MAIQL}) to updates in an equivalent continuous-time Ordinary Differential Equation (ODE). This conversion then allows us to use the analysis tools created to analyse the evolution of two-timescale ODEs to show that our coupled updates converge. The proof detailed below broadly follows along the lines of \citet{avrachenkov2020whittle}, but where they discuss convergence in the binary action case, we generalize their proof to the multi-action scenario by using the notion of multi-action indexability from \cite{glazebrook2006some}.

\MAIQL*

\begin{proof}
    To convert these discrete updates to ODEs, we map a given time-step $t$ to a point $\tau = T(t)$ in a continuous time, such that any time $T(t) = \sum_{m=0}^t \alpha(t)$. Because we're parameterising the time with $\alpha$ (rather than $\gamma$) we call $\tau$ the fast time-scale. To make this more concrete, we define $Q(\tau)$ as a function of the Q-value with time, and set $Q(T(t)) = Q^t$ to the value of the Q-function after $t$ updates
    . Then, for values of $T(t) < \tau < T(t+1)$, $Q(\tau)$ is assumed to be linearly interpolated between $Q^t$ and $Q^{t+1}$, creating a continuous function of $\tau$. Similarly, we define $\lambda(\tau)$ such that $\lambda(T(t)) = \lambda^t$
    
    We can then re-arrange the terms in Equation \ref{eqn:maqupdate} to create an ODE that characterises the value of $Q(\tau)$:
    \begin{align*}
        &Q^{t+1}(s, a_j) = Q^t(s, a) + \alpha(t) \big [ [r(s) - \lambda^{t}_{s,a_j}c_j - f(Q^t)  \\
        &\qquad \qquad \qquad \qquad \qquad \qquad +\max_{a_j'\in \{0, 1\}} Q^t(s', a_j')] - Q^t(s, a_j) \big ] \\
        \Rightarrow &\underbrace{\frac{Q^{t+1}(s, a_j) - Q^t(s, a_j)}{\alpha(t)}}_{\dot{Q}(\tau)} = [r(s) - \lambda^{t}_{s,a_j}c_j - f(Q^t)  \\[-2.5em]
        &\qquad \qquad \qquad \qquad \qquad \qquad+ \max_{a_j'\in \{0, 1\}} Q^t(s', a_j')] - Q^t(s, a_j)
    \end{align*}
    where $\dot{Q}(\tau)$ is the derivative of $Q(\tau)$ and corresponds to the slope of the interpolated function in the range $(T(t), T(t+1))$.
    
    \noindent Similarly, we can re-arrange Equation \ref{eqn:malamupdate} to get the ODE for $\lambda(\tau)$:
    \begin{align}
        & \; \lambda_{s,a_j}^{t+1} = \lambda_{s,a_j}^t + \alpha(t) \left ( \frac{\gamma(t)}{\alpha(t)} \right ) (Q^t(s, a_j) - Q^t(s, a_{j-1})) \nonumber \\ \label{eqn:fasttslam}
        \Rightarrow & \underbrace{\frac{\lambda_{s,a_j}^{t+1} - \lambda_{s,a_j}^t}{\alpha(t)}}_{\dot{\lambda}(\tau)} =  \left ( \frac{\gamma(t)}{\alpha(t)} \right ) (Q^t(s, a_j) - Q^t(s, a_{j-1}))
    \end{align}
    
    Then, if look at Equation \ref{eqn:fasttslam}, we see $\lim_{\tau \to \infty} \dot{\lambda}(\tau) \to 0$
    because, by Assumption \ref{ass:alpha} (c), $\lim_{t \to \infty} \frac{\gamma(t)}{\alpha(t)} \to 0$ and, by Assumption \ref{ass:alpha} (A), $T(\infty) = \sum_{t=0}^{\infty} \alpha(t) \to \infty$. Therefore, $\lambda(\tau)$ can be seen as quasi-static w.r.t. $Q(\tau)$ at the fast time-scale. As a result, the updates in this time-scale correspond to standard Q-Learning for a fixed MDP defined by the value of $\lambda(\tau)$. Given Assumptions \ref{ass:unichain}, \ref{ass:alpha} (A), and \ref{ass:async}, this is known to converge to the optimal Q-values $Q^*_\lambda$ for the given value of $\lambda(\tau)$ \cite{abounadi2001learning}.
    
    Now, at the slow time-scale $\tau'$, we can repeat this continuous-time re-parameterisation, except with $T'(t) = \sum_{m=0}^t \gamma(t)$. Then, re-arranging Equation \ref{eqn:maqupdate} in a similar way as above, we get:
    \begin{align*}
        &\underbrace{\frac{Q^{t+1}(s, a_j) - Q^t(s, a_j)}{\gamma(t)}}_{\dot{Q}(\tau')} =  \left ( \frac{\alpha(t)}{\gamma(t)} \right ) [r(s) - \lambda^{t}_{s,a_j}c_j - f(Q^t)  \\[-2.5em]
        &\qquad \qquad \qquad \qquad \qquad \qquad \hspace{3mm} + \max_{a_j'\in \{0, 1\}} Q^t(s', a_j')] - Q^t(s, a_j)
    \end{align*}

    Now, given that $\lim_{t \to \infty} \frac{\alpha(t)}{\gamma(t)} \to \infty$, and from the argument above about the Q-values converging in the fast time-scale, we can see the interpolated $\lambda(\tau')$ value as tracking the converged Q-values $Q^*_{\lambda(\tau')}$ (for that value of $\lambda(\tau')$). Then, we can write the ODE for $\lambda(\tau')$ as:
    $$\dot{\lambda}(\tau') = Q^*_{\lambda(\tau')}(s, a_j) - Q^*_{\lambda(\tau')}(s, a_{j-1})$$
    where $Q^*_{\lambda(\tau')}$ corresponds to the optimal Q-values corresponding to the given value of $\lambda(\tau')$.
    
    Now, if $\lambda(\tau') < \lambda^*_{s, a_j}$ (the multi-action index for state $s$ and action $a_j$), by the definition of the multi-action index from the main text, we know that an action of weight $c_j$ or higher is preferred. As a result, we see that $\dot{\lambda}(\tau') > 0$ in that case. If $\lambda(\tau') > \lambda^*_{s, a_j}$, the opposite is true and so $\dot{\lambda}(\tau') < 0$. Then, because $\lambda(0) = 0$ is bounded and given the step-sizes in Assumption \ref{ass:alpha} (B), $\lambda(\tau')$ converges to an equilibrium in which $Q^*_{\lambda}(s, a_j) - Q^*_{\lambda}(s, a_{}) \to 0$. 
    
    \emph{Given that, by definition, $\lambda^*(s, a_j)$ is the value at which $Q^*_{\lambda}(s, a_j) = Q^*_{\lambda}(s, a_{j-1})$, $\lambda(\tau')$ converges to the multi-action index.}
\end{proof}

This is a high-level proof, but the specific conditions for convergence can be seen in \citet{lakshminarayanan2017stability}. They require 5 conditions: (1) Lipschitzness, (2) Bounded `noise', (3) Properties about the relative step-sizes, (4) Convergence of fast time-scale, and (5) Convergence of slow time-scale.

Of these, (1)-(4) proceed in much the same way as in \citet{avrachenkov2020whittle} because they do not depend on the multi-action extension of indexability. In addition, it is easy to show that the proof of (5) from \citet{avrachenkov2020whittle} extends to the multi-action case which considers the limiting value of $Q(s,a_j) - Q(s,a_{j-1})$ rather than $Q(s,1) - Q(s,0)$. As a result, we refer the reader to \citet{avrachenkov2020whittle} for the complete proof.

\section{Reproducibility}
Code is available at \url{https://github.com/killian-34/MAIQL_and_LPQL}. 
All the Q and $\lambda$ values are initiated to zero in all the experiments. The parameter settings used for the two process type, random, and medication adherence data experiments are included in Tables \ref{table:two_process}, \ref{table:random}, and \ref{table:med_adherence}, respectively. $C$ is the multiplier for the size of the Q-value updates. $C^\prime$ is the multiplier for the size of the index value updates. ``Rp/dream'' is the number of replays per dream. ``Rp T'' is the replay period (replay every T steps). $\lambda$-bound is the upper bound (and negative of the lower bound) imposed on values of the indices for WIBQL and MAIQL during learning -- placing these bounds sometimes helps prevent divergent behavior in early rounds when updates are large -- $\lambda_{\max}$ is the upper bound value that an index could take, as defined by the problem parameters, i.e.,  $\frac{\max\{r\}}{\min\{\mathcal{C}\} (1-\beta)}$ \cite{killian2021multiAction}. $D$ is the divisor of the decaying $\epsilon$-greedy function as well as the divisor of $\alpha(t)$ and $\gamma(t)$, the decaying functions defining the size of the updates of Q-values and index values, defined in the previous section. $\epsilon_0$ is the multiplier for the $\epsilon$-greedy function. $n_{lam}$ is the number of points in $\lambda$-space used to approximate the Q(s, a, $\lambda$)-functions in LQPL and MAIQL-Aprx. All values were determined via manual tuning -- empirically we found that most parameter settings led to similar long-term performance between algorithms, as long as the settings did not cause the algorithms to diverge. In the tables, M-Aprx stands for MAIQL-Aprx.

\begin{table}[t]
\resizebox{\columnwidth}{!}{%
\begin{tabular}{|l|l|l|l|l|l|}
\hline
                & WIBQL    & QL-$\lambda$=0 & MAIQL & M-Aprx & LPQL     \\ \hline
C               & 0.1      & 0.2            & 0.1   & 0.4        & 0.4      \\ \hline
$C^\prime$      & 0.2      & -              & 0.2   & -          & -        \\ \hline
Rp/dream      & NA       & 1000           & 1000  & 1000       & NA       \\ \hline
Rp T          & 1E+06 & 100            & 10    & 100        & 1E+06 \\ \hline
$\lambda$-bound & 3        & -              & 3     & 3          & 3        \\ \hline
D               & 500      & 500            & 500   & 500        & 500      \\ \hline
$\epsilon_0$   & 0.99      & 0.99            & 0.99   & 0.99        & 0.99      \\ \hline
$n_{lam}$       & -        & -              & -     & 3000       & 3000     \\ \hline
\end{tabular}
}
\caption{Parameter settings for two process experiment.}
\label{table:two_process}
\end{table}

\begin{table}[t]
\resizebox{\columnwidth}{!}{%
\begin{tabular}{|l|l|l|l|l|l|}
\hline
                & WIBQL & QL-$\lambda$=0 & MAIQL & M-Aprx & LPQL     \\ \hline
C               & -     & -              & 0.2   & 0.8        & 0.8      \\ \hline
$C^\prime$      & -     & -              & 0.4   & -          & -        \\ \hline
Rp/dream      & -     & -              & 1000  & NA         & NA       \\ \hline
Rp T          & -     & -              & 100   & 1E+06   & 1E+06 \\ \hline
$\lambda$-bound & -     & -              & $\lambda_{\max}$  & $\lambda_{\max}$       & $\lambda_{\max}$     \\ \hline
D               & -     & -              & 500   & 500        & 500      \\ \hline
$\epsilon_0$   & -      & -            & 0.99   & 0.99        & 0.99      \\ \hline
$n_{lam}$       & -     & -              & -     & 2000       & 2000     \\ \hline
\end{tabular}
}
\caption{Parameter settings for random data experiment.}
\label{table:random}
\end{table}

\begin{table}[t]
\resizebox{\columnwidth}{!}{%
\begin{tabular}{|l|l|l|l|l|l|}
\hline
                & WIBQL & QL-$\lambda$=0 & MAIQL & M-Aprx & LPQL \\ \hline
C               & -     & 0.8            & 0.05  & 0.8        & 0.8  \\ \hline
$C^\prime$      & -     & -              & 0.1   & -          & -    \\ \hline
Rp/dream      & -     & 1000           & 1000  & 1000       & 1000 \\ \hline
Rp T          & -     & 10              & 5     & 5          & 5    \\ \hline
$\lambda$-bound & -     & -              & $\lambda_{\max}$  & $\lambda_{\max}$       & $\lambda_{\max}$ \\ \hline
D               & -     & 1000           & 2000  & 1000       & 1000 \\ \hline
$\epsilon_0$   & -      & 0.99            & 0.99   & 0.99        & 0.99      \\ \hline
$n_{lam}$       & -     & -              & -     & 2000       & 2000 \\ \hline
\end{tabular}
}
\caption{Parameter settings for adherence data experiment.}
\label{table:med_adherence}
\end{table}

\section{Algorithm Pseudocodes}
See Algorithms \ref{alg:maiql_update} and \ref{alg:maiql_act} for the update and action selection steps of MAIQL and Algorithms \ref{alg:lpql_update} and \ref{alg:lpql_act} for the update and action selection steps of LPQL. The linear program for Oracle-LP-Index for a given current state $\bm{s}_{cur}$ and action $a_k$ is given below:

\begin{equation}
\begin{aligned}\label{eq:oracle_lp_index_lp}
    &\min_{V^i(s^i, \lambda^i), \lambda^i} \sum_{i=0}^{N-1}\frac{\lambda^i B}{1-\beta} + \sum_{i=0}^{N-1}\mu^i(s^i) V^i(s^i, \lambda^i) \\
    &\text{s.t. }V^i(s^i, \lambda^i) \ge r^i(s^i) - \lambda^i c_j + \beta\sum_{s^{i\prime}}T(s^i, a_j^i, s^{i\prime}) V^i(s^{i\prime},\lambda^i) \\
    & \hspace{24mm} \forall i \in \{0,...,N-1\},\hspace{2mm} \forall s^i \in \mathcal{S},\hspace{2mm} \forall a_j \in \mathcal{A} \\
    &\hspace{5mm}r^i(s^i_{cur}) - \lambda^i c_k + \beta\sum_{s^{i\prime}}T(s_{cur}^i, a_k^i, s^{i\prime}) V^i(s^{i\prime},\lambda^i) = \\
    &\hspace{10mm} r^i(s^i_{cur}) - \lambda^i c_{k-1} + \beta\sum_{s^{i\prime}}T(s_{cur}^i, a_{k-1}^i, s^{i\prime}) V^i(s^{i\prime},\lambda^i) \\
    & \hspace{52mm} \forall i \in \{0,...,N-1\} \\
    &\hspace{5mm}\lambda^i \ge 0 \hspace{2mm} \forall i \in \{0,...,N-1\}
\end{aligned}
\end{equation}

\noindent The LP is similar to Eq.~\ref{eq:the_lp}, but differs in two ways. First, instead of having a single $\lambda$ value across all arms, each arm has its own independent $\lambda^i$ value. Second, the second group of constraints is new, and forces the $\lambda^i$ values to be set such that the planner would be indifferent between taking the action in question $a_k$ or the action that is one step cheaper $a_{k-1}$, which follows exactly the definition of the multi-action indexes. Note that although the indexes can each be computed independently, for convenience, we compute the index for a given $a_k$ for each arm simultaneously to reduce overhead, as given in the above LP.

The \textsc{ActionKnapsackILP} referenced in Algorithm \ref{alg:lpql_act} is the same as the modified knapsack given in \citet{killian2021multiAction}, reproduced below:

\begin{equation}
\begin{aligned}
    &\max_{\bm{A}} \sum_{i=0}^{N-1}\sum_{j=0}^{M-1}\bm{A}_{ij}Q_{\bm{s},\lambda_{ind}}(i, a_j) \\
    &\text{s.t. }\sum_{i=0}^{N-1}\sum_{j=0}^{M-1} \bm{A}_{ij}c_j \le B \\
    &\sum_{j=0}^{M-1} \bm{A}_{ij} = 1 \hspace{3mm} \forall i \in 0,\ldots,N-1 \\
    &\bm{A}_{ij} \in \{0,1\}
    \label{eq:knapsack}
\end{aligned}
\end{equation}
\noindent where $Q_{\bm{s},\lambda_{ind}}(i, a_j)$ is the $Q$-function for each arm filtered to the current state of the arms, $\bm{s}$, and minimizing value $\lambda_{ind}$, as given by the penultimate line of Algorithm \ref{alg:lpql_act}.

\begin{algorithm}[b!]
\DontPrintSemicolon
\SetAlgoLined
\SetNoFillComment
\LinesNotNumbered
\KwData{
$Q \in \mathbb{R}^{N\times|\mathcal{S}|\times(|\mathcal{A}|-1)\times|\mathcal{S}|\times|\mathcal{A}|}$, \tcp*{Need one copy of $Q[s,a]$ for each index on each arm} \
\hspace{1mm}$\lambda\in\mathbb{R}^{N\times|\mathcal{S}|\times(|\mathcal{A}|-1)},$ \tcp*{multi-action index estimates} \
\hspace{1mm}$\text{Batch}, \mathcal{C},$ \tcp*{Experience tuples, action costs} \
\hspace{1mm}$t, \nu(\cdot),$ \tcp*{iteration, state-action counter} \ 
\hspace{1mm}$\mathcal{S}, \mathcal{A}, N$ \tcp*{state space, action space, \# of arms}
}
\textbf{Hyperparameters:} $\beta, C, C^{\prime}, D$ \tcp*{See maintext} \
\LinesNumbered
\For{$(n, s, a, r, s^\prime) \in \text{Batch}$}{
    
    $\alpha = \frac{C}{\lceil \frac{\nu(s,a,n)}{D} \rceil}$ \;

    \For{$i \in 0, \ldots, |\mathcal{S}|$}{
        \For{$j \in 1, \ldots, |\mathcal{A}|$}{
            $Q[n,i,j,s,a] \mathrel{+}= \alpha(r - \mathcal{C}[a]*\lambda[i, j] + \beta*\max\{Q[n,i,j,s^\prime]\} - Q[n,i,j,s,a])$ \;
        }
    }
    
    \If{$a \neq 0 \And t\Mod{N}==0$}
    {
        $\gamma = \frac{C^{\prime}}{1 + \lceil \frac{\nu(s,a,n)\log{\nu(s,a,n)}}{D} \rceil}$ \;
    
        $\lambda[s, a] \mathrel{+}= \frac{\gamma(Q[n,s,a,s,a]) - Q[n,s,a,s,a-1])}{\mathcal{C}[a] - \mathcal{C}[a-1]}$ \;
    }
}
 \Return $Q$, $\lambda$
 \caption{MAIQL Update}
 \label{alg:maiql_update}
\end{algorithm}

\begin{algorithm}[t!]
\DontPrintSemicolon
\SetAlgoLined
\SetNoFillComment
\LinesNotNumbered
\KwData{$\lambda\in\mathbb{R}^{N\times|\mathcal{S}|\times(|\mathcal{A}|-1)},$ \tcp*{multi-action index estimates} \
$\hspace{1mm} \bm{s}\in\mathbb{R}^N$ \tcp*{current state of all arms} \
$\hspace{1mm} t, N, B$ \tcp*{current iteration, \# of arms, budget}
}
\If{$\textsc{EpsilonGreedy}(t)$}
{
    \Return $\textsc{RandomAction}()$\;
}
\Else{
$\bm{a} = [0 \text{ for \_ in range}(N)]$\;
$\lambda_{f} = \textsc{FilterCurrentState}(\lambda,\bm{s})$ \tcp*{$\lambda_{f}\in\mathbb{R}^{N\times(|\mathcal{A}|-1)}$} \
\For{$i \in 0 \ldots B$}{
    $i = \argmax(\lambda_{f}[\bm{a}+1] - \lambda_{f}[\bm{a}])$ \tcp*{$\bm{a}$ is a vector index, $\argmax$ ignores out of bounds indexes}
    $a[i] \mathrel{+}= 1$ \;
}
}

 \Return $\bm{a}$
 \caption{MAIQL Action Select}
 \label{alg:maiql_act}
\end{algorithm}

\begin{algorithm}[t]
\DontPrintSemicolon
\SetAlgoLined
\SetNoFillComment
\LinesNotNumbered
\KwData{
$Q \in \mathbb{R}^{N\times n_{lam} |\mathcal{S}|\times|\mathcal{A}|}$, \tcp*{Need one copy of $Q[s,a]$ for each of the $n_{lam}$ test points on each arm} \
\hspace{1mm}$\text{Batch}, \mathcal{C},$ \tcp*{Experience tuples, action costs} \
\hspace{1mm}$\lambda_{\max},$ \tcp*{Max $\lambda$ at which to estimate $Q$} \
\hspace{1mm}$n_{lam},$ \tcp*{\# of $\lambda$ points at which to estimate $Q$} \
\hspace{1mm}$\nu(\cdot)$ \tcp*{state-action counter}
} 
\textbf{Hyperparameters:} $\beta, C, D$ \tcp*{See maintext} \
\For{$(n, s, a, r, s^\prime) \in \text{Batch}$}{
    
    $\alpha = \frac{C}{\lceil \frac{\nu(s,a,n)}{D} \rceil}$

    \For{$i \in 0,\ldots,n_{lam}$}{
        $\lambda_p = \frac{i*\lambda_{\max}}{n_{lam}}$ \;
        $Q[n,i,s,a]\mathrel{+}=\alpha(r - \mathcal{C}[a]*\lambda_p + \beta*\max\{Q[n,i,s^\prime]\} - Q[n,i,s,a])$
    }
}
 \Return Q
 \caption{LPQL Update}
 \label{alg:lpql_update}
\end{algorithm}

\begin{algorithm}[t]
\DontPrintSemicolon
\SetAlgoLined
\SetNoFillComment
\LinesNotNumbered
\KwData{$Q \in \mathbb{R}^{N\times n_{lam} |\mathcal{S}|\times|\mathcal{A}|}$, \tcp*{$Q$-functions for each of the $n_{lam}$ test points on each arm} \
$\hspace{1mm} \bm{s}\in\mathbb{R}^N$ \tcp*{current state of all arms} \
\hspace{1mm}$\lambda_{\max},$ \tcp*{Max $\lambda$ at which $Q$ is estimated} \
\hspace{1mm}$n_{lam},$ \tcp*{\# of $\lambda$ points at which $Q$ is estimated} \
$\hspace{1mm} t, \beta$ \tcp*{iteration, discount factor} \
$\hspace{1mm} N, \mathcal{C}, B$ \tcp*{\# of arms, action costs, budget} 
}
\If{$\textsc{EpsilonGreedy}(t)$}
{
    \Return $\textsc{RandomAction}()$\;
}
$\bm{a} = [0 \text{ for \_ in range }(N)]$\;
$Q_f = \textsc{FilterCurrentState}(Q,\bm{s})$ \tcp*{$Q_f \in \mathbb{R}^{N\times n_{lam}\times|\mathcal{A}|}$}\
$\lambda_{ind}=-1$ \;
\tcc{The min of Eq.~\ref{eq:the_lp} occurs at the point where the negative sum of slopes of all $V^i=\max\{Q^i_\lambda\}$ is $\le$ $B/(1-\beta)$, so we will iterate through our estimates of $Q^i_\lambda$ and stop our search at the first point where that is true.}
\For{$i \in 0,\ldots,n_{lam}$}{
    $\lambda_p^0 = \frac{i*\lambda_{\max}}{n_{lam}}$ \;
    $\lambda_p^1 = \frac{(i+1)*\lambda_{\max}}{n_{lam}}$ \;
    $m_V = \frac{\max_a\{Q_f[:,i+1]\} - \max_a\{Q_f[:,i]\}}{\lambda_p^1 - \lambda_p^0}$ \tcp*{$m_V \in \mathbb{R}^N$} \
    \If{$ \sum_n\{m_V\} \ge -\frac{B}{1-\beta}$}
    {
        $\lambda_{ind}=i$ \;
        \text{break}
    }
}

$\bm{a} = \textsc{ActionKnapsackILP}(Q_f[:,\lambda_{ind},:],\mathcal{C}, B)$

 \Return $\bm{a}$
 \caption{LPQL Action Select}
 \label{alg:lpql_act}
\end{algorithm}

\textsc{RandomAction}, referenced in Algorithms \ref{alg:maiql_act} and \ref{alg:lpql_act}, chooses random actions through the following iterative procedure: (1) randomly choose an arm with uniform probability, (2) randomly choose an action with probability inversely proportional to one plus its cost (must add one to avoid dividing by 0 for no-action). The procedure iterates until the budget is exhausted.

$\textsc{EpsilonGreedy}(t)$, also referenced in Algorithms \ref{alg:maiql_act} and \ref{alg:lpql_act}, draws a uniform random number between 0 and 1 and returns true if it is less than $\epsilon_0/\ceil*{\frac{t}{D}}$ and false otherwise.

\section{Medication Adherence Setting Details}
We used the following procedure to estimate transition probabilities from the medication adherence data from \citet{killian2019learning}. First, we specify a history length of $L$. This gives a state space of size $2^L$ for each arm. Then, for each patient in the data, we count all of the occurrences of each state transition across a treatment regimen of 6 months (168 days). If $L$ was small (e.g., 1 or 2), we could take a frequentist approach and simply normalize these counts appropriately to get valid transition probabilities to sample for experiments. However, as the history length $L$ gets larger, the number of non-zero entries in the count data for state transitions become large. We take two steps to account for this sparsity. (1) We run $K$-means clustering over all patients, using the count data as features, then combine the counts for all patients within a cluster. Intuitively, the larger the $K$, the more ``peaks'' of the distribution of patient adherence modes we will try to approximate, but the fewer data points are available to estimate the distribution in each cluster --- however, it may be desirable to have more clusters to allow for some samples to come from uncommon but ``diverse'' modes that may be challenging to plan for. In this paper, we set $K$ to 10. (2) We then take a Bayesian approach, rather than a frequentist approach for sampling patients/processes from the clustered counts data. That is, we treat the counts as priors of a beta distribution, then sample transition probabilities from those distributions according to the priors. Finally, to simulate action effects, since actions were not recorded in the available adherence data, we scale the priors multiplicatively according to the index of the action, i.e., larger/more expensive actions increase the priors associated with moving to the adhering state. 

\balance

In summary, to get a transition function for a single simulated arm in the medication adherence experimental setting, we do the following. First, randomly choose a cluster, with probability weighted by the number of patients in the cluster. Then, build up a transition matrix by sampling each row according to its own beta distribution with priors given by the counts data (i.e., actual observations of $s\rightarrow{}s^\prime$ transitions), scaled by the action effects. 

This process was desirable for producing simulated arms with transition functions tailored to resemble that of a real world dataset, while allowing for some randomness via the sampling procedure, as well as a straightforward way to impose simulated action effects. However, one downside of this approach is that, since each row of the transition matrix is sampled independently, this may produce simulated arms whose probability of adherence changes in a non-smooth manner as a function of history. For example, in the real-world, we would expect that $P(0111\rightarrow{}1111)$ is correlated with $P(1011\rightarrow{}0111)$ and that $P(0000\rightarrow{}0000)$ is correlated with $P(1000\rightarrow{}0000)$, but our procedure would not necessarily enforce these relationships if there were not sufficient occurrences of each transition in the counts data.

The python code used to execute this procedure is included in the repository at \url{https://github.com/killian-34/MAIQL_and_LPQL}.




\end{document}